\documentclass[conference]{IEEEtran}
\IEEEoverridecommandlockouts
\usepackage{amsmath,amssymb,amsfonts, amsthm}
\DeclareMathOperator*{\argmin}{arg\,min}

\usepackage{algorithmic}
\usepackage[numbers]{natbib}\usepackage{graphicx}\usepackage{adjustbox}
\usepackage{textcomp}
\usepackage{url}
\usepackage{xcolor}
\usepackage{multirow}
\usepackage{booktabs}
\usepackage{placeins}
\usepackage{float}
\usepackage{caption}

\newtheorem{theorem}{Theorem}
\newtheorem{lemma}{Lemma}

\theoremstyle{remark}

\usepackage{enumitem}

\usepackage{tikz}
\usetikzlibrary{arrows.meta,positioning,fit,shapes.multipart}
\def\BibTeX{{\rm B\kern-.05em{\sc i\kern-.025em b}\kern-.08em
    T\kern-.1667em\lower.7ex\hbox{E}\kern-.125emX}}

\newcommand{\ali}[1]{}
\newcommand{\wzl}[1]{}
\newcommand{\AppendixTablesRef}{C}  
    
\begin{document}

\title{Transfer Learning for Meta-analysis Under Covariate Shift\thanks{Code available at: \protect\url{https://github.com/wangzilongri/SparseTLDRICHI2026}}}
\author{\IEEEauthorblockN{Zilong Wang}
\IEEEauthorblockA{\textit{Industrial and Systems Engineering} \\
\textit{Georgia Institute of Technology} \\
Atlanta, Georgia, USA \\
\IEEEauthorrefmark{1} zwang937@gatech.edu}
\and
\IEEEauthorblockN{Ali Abdeen}
\IEEEauthorblockA{\textit{Industrial and Systems Engineering} \\
\textit{Georgia Institute of Technology} \\
Atlanta, Georgia, USA \\
\IEEEauthorrefmark{2} aabdeen7@gatech.edu}
\and
\IEEEauthorblockN{Turgay Ayer}
\IEEEauthorblockA{\textit{Industrial and Systems Engineering} \\
\textit{Georgia Institute of Technology} \\
Atlanta, Georgia, USA \\
\IEEEauthorrefmark{3} ayer@isye.gatech.edu}
}




\maketitle

\begin{abstract}
Randomized controlled trials often do not represent the populations where decisions are made, and covariate shift across studies can invalidate standard IPD meta-analysis and transport estimators. We propose a placebo-anchored transport framework that treats source-trial outcomes as abundant proxy signals and target-trial placebo outcomes as scarce, high-fidelity gold labels to calibrate baseline risk. A low-complexity (sparse) correction anchors proxy outcome models to the target population, and the anchored models are embedded in a cross-fitted doubly robust learner, yielding a Neyman-orthogonal, target-site doubly robust estimator for patient-level heterogeneous treatment effects when target treated outcomes are available. We distinguish two regimes: in connected targets (with a treated arm), the method yields target-identified effect estimates; in disconnected targets (placebo-only), it reduces to a principled screen--then--transport procedure under explicit working-model transport assumptions. Experiments on synthetic data and a semi-synthetic IHDP benchmark evaluate pointwise CATE accuracy, ATE error, ranking quality for targeting, decision-theoretic policy regret, and calibration. Across connected settings, the proposed method is best or near-best and improves substantially over proxy-only, target-only, and transport baselines at small target sample sizes; in disconnected settings, it retains strong ranking performance for targeting while pointwise accuracy depends on the strength of the working transport condition. 
\end{abstract}

\section{Introduction}
Randomized controlled trials (RCTs) remain the gold standard for estimating treatment effects \citep{friedman2015fundamentals}, and drug approvals \citep{FDA1998}. 
Yet, their external validity or generalizability remains a major challenge,as trial participants often differ from the patient population of interest \citep{cartwright2007rcts}, thus decision-makers often need to know how a therapy would perform in a different target population than the one studied. 
This challenge arises frequently in oncology and other therapeutic areas where multiple trials are conducted in distinct populations, and where no single trial provides a direct comparison in the population of interest. \\

Network meta-analysis (NMA) has become the dominant framework for combining this evidence, allowing estimation of relative treatment effects across a connected network of trials \citep{lu2004combination,salanti2012indirect}.
When individual patient data (IPD) are available, IPD-NMA can incorporate treatment–covariate interactions to explore the heterogeneity of the effect and, in principle, estimate treatment effects in specific subpopulations \cite{riley2021individual}. 
However, standard IPD-NMA methods rely on two key conditions: Network connectivity – the target trial must be connected to the network through a shared comparator arm (e.g., placebo or standard of care). 
Shared-comparator comparability – the comparator arms across trials must be exchangeable after adjustment for measured. \cite{riley2023using}. \\ 
In practice, these conditions are often violated. 
Network connectivity is violated if target trials are disconnected, for example containing only a placebo arm for the treatment of interest. 
Baseline risks in shared comparators differ across studies due to covariate shift and unmeasured effect modifiers, violating shared-comparator comparability.\\
IPD-NMA typically models relative effects without anchoring to the observed absolute risk in the target population, leaving results vulnerable to bias when proportional hazards do not hold or when hazard ratios exhibit non-collapsibility. 
Furthermore, standard NMA outputs are network-wide averages; transporting them to a specific target cohort usually requires meta-regression or marginalization, not patient-level counterfactual prediction.

These limitations create a methodological gap: there is no established approach that can
(i) transport trial evidence to a target population lacking a concurrent treated arm,
(ii) correct for covariate shift in both measured covariates and residual baseline risk, and
(iii) support patient-level, target-specific decision analysis in disconnected trial settings. \\

We propose a placebo-anchored transport framework to address this gap.
Our method learns conditional outcome models from source-trial IPD (treatment and placebo),
explicitly calibrates baseline risk using observed placebo outcomes in the target trial,
and embeds the calibrated models in an orthogonal estimation pipeline.
This approach requires no network connectivity and directly adjusts for covariate shift through anchoring.
When treated outcomes are observed in the target (even if limited), it yields individualized
counterfactual predictions via a doubly robust learner; when treated outcomes are unavailable,
it provides calibrated baseline risk and a principled working-model extrapolation for
scenario and sensitivity analysis.

\section{Related Work}

IPD meta-analysis and transportability methods provide the main alternatives for combining trial evidence. Standard IPD-NMA \cite{riley2021individual,lu2004combination} requires network connectivity and shared-comparator comparability; when the target is disconnected (e.g., placebo-only) or under covariate shift, these conditions fail. Transportability and weighting approaches \cite{dahabreh2019generalizing,pearl2022external} formalize external validity via selection diagrams and reweighting (IPW \cite{horvitz1952generalization,rosenbaum1983central}, AIPW \cite{robins1995semiparametric}, entropy balancing \cite{hainmueller2012entropy}), but rely on overlap, correct propensity specification, and do not anchor to target placebo outcomes. None of these directly address calibration of proxy outcome models using scarce target placebo data in disconnected settings.
\begin{itemize}[noitemsep,topsep=2pt]
\item \textbf{IPD-NMA and transportability:} Network meta-analysis and transport estimators \cite{riley2021individual,dahabreh2019generalizing} assume connectivity or sufficient overlap; they do not use placebo-at-target as a gold calibration signal.
\item \textbf{Proxy--gold and transfer learning:} High-dimensional transfer learning \cite{bastani2021predicting,tian2023transfer} treats abundant source data as proxy and scarce target data as gold; we adopt this view and anchor to target placebo.
\item \textbf{Orthogonal and DR learners:} Doubly robust and DML estimators \cite{chernozhukov2018double,kennedy2024semiparametric} give Neyman-orthogonal CATE estimation; we combine placebo-anchored outcome models with cross-fitted DR learning.
\item \textbf{Other approaches:} Classical parametric IPD meta-analysis, ML (e.g., causal forests \cite{athey2019generalized}, BART\cite{Chipman_2010}), and TMLE \cite{van2011targeted} improve flexibility but typically assume connectivity or lack explicit placebo anchoring; deep causal methods exist but share similar limitations.
\end{itemize}
We therefore propose a placebo-anchored transport framework that calibrates proxy models to target placebo outcomes and embeds them in an orthogonal DR pipeline, without requiring network connectivity.

\subsection{Our Contributions}

This paper develops a calibrated framework for transporting individualized treatment effect estimates across heterogeneous randomized controlled trials under covariate shift. We formulate the transport problem as a \emph{proxy--gold} learning task in the sense of high-dimensional transfer learning, where outcomes from source trials provide abundant but potentially miscalibrated proxy signals, and placebo outcomes observed at the target site provide a limited but informative calibration signal for baseline risk. This formulation does not claim causal identification of target treatment effects from target data alone; instead, it defines a stable working-model estimand suitable for disconnected or weakly connected trial settings.

Within this framework, we introduce a placebo-anchored correction mechanism that adjusts proxy outcome regressions using a restricted, low-complexity function class. The resulting anchoring step explicitly targets systematic baseline miscalibration induced by population heterogeneity and yields calibrated outcome predictions in the target population. The low-complexity restriction is imposed as a regularity condition rather than an identification assumption, allowing estimation error to be decomposed into stochastic error and an explicit structural approximation term.

We further embed the anchoring procedure in a doubly robust learner with cross-fitting, exploiting the fact that treatment assignment is randomized and propensities are known by design in randomized trials. In connected targets, where both treatment arms are observed at the target site, the resulting estimator is Neyman-orthogonal and admits a pointwise asymptotic linear expansion for the identified target-site CATE in connected targets. Source data contribute by improving the first-stage nuisance rates through transfer learning, while finite-sample performance depends jointly on transferable source sample sizes, the target calibration sample size, and approximation error under the working model. In disconnected targets, where treated outcomes are unavailable at the target, we instead derive a transport error decomposition around an oracle transported target, separating estimation error on selected sources, structural transport bias, and placebo-screening error. Proofs and detailed rate and error statements appear in the appendix (Section~\ref{sec:appendix_proof}).

Finally, through synthetic and semi-synthetic experiments, including ablations and robustness checks, we empirically characterize the operating regimes of placebo anchoring and illustrate how performance degrades gracefully toward proxy-only baselines as structural assumptions are weakened.

\subsection{Organization of this Paper}
The remainder of the paper proceeds as follows. 
\S~\ref{sec:P3_Preliminaries} formalizes the problem setting, data structure, and identification assumptions. 
\S~\ref{sec:P3_Proposed_Framework} presents the placebo-anchored proxy–gold estimator and the DR-Learner integration in detail. 
\S~\ref{sec:P3_experiments} evaluates the approach on synthetic ablations, along with robustness checks when some identification assumptions are violated. 
\S~\ref{sec:P3_ihdp} evaluates the approach on a semi-synthetic benchmark derived from the IHDP dataset.
Finally, \S~\ref{sec:P3_Conclusion} summarizes findings, discusses limitations, and outlines directions for future work.

\section{Preliminaries}\label{sec:P3_Preliminaries}
Here we introduce our preliminaries and notation for this paper.

\subsection{Problem Statement}
We study the setting where a therapy is evaluated across multiple RCTs, each conducted at a different site or in a different patient population. Within every site, both treatment and placebo arms are observed under a harmonized protocol (identical dosage, randomization scheme, and eligibility criteria). 
Consequently, treatment assignment is randomized and the propensity model is known by design. 

The methodological challenge arises because the covariate distribution of enrolled patients drifts across sites, reflecting population heterogeneity (e.g., the same oncology drug trialed in the U.S. versus Norway). 
Such covariate shift induces systematic differences in baseline risk and treatment–covariate interactions, leading to biased estimates if outcomes are naively pooled or if network connectivity assumptions are imposed. 

Our goal is to estimate patient-level counterfactual outcomes and treatment effects in a target site by leveraging information from source trials while anchoring estimates to the observed placebo outcomes in the target. 
Following \cite{bastani2021predicting}'s proxy–gold paradigm, we treat source-site outcomes as \texttt{proxy} labels and the target-site placebo outcomes as scarce but high-fidelity \texttt{gold} labels. 
This enables calibrated transport of treatment effects across heterogeneous populations without requiring network connectivity.

\subsection{Data and Proxy--Gold Setup}\label{sec:data_proxy_gold}
We observe multiple randomized controlled trials (RCTs) indexed by $c\in\{1,2,\dotsc,C\}$ (\texttt{sources}) together with one \texttt{target} trial, indexed by $c=0$. In each trial $c$, let $X\in\mathbb{R}^p$ denote baseline covariates, $A\in\{0,1\}$ the randomized treatment assignment, and $Y$ the outcome. Because treatment is randomized within each trial, the propensity $e_c(x)=\mathbb{P}(A=1\mid X=x,c)$ is known by design.

\textbf{Proxy data (abundant).}  
From the source trials we have individual participant data (IPD) for both treatment and placebo arms,
\[
\mathcal{D}_\mathrm{proxy} \;=\;\bigcup_{c=1}^C \{(X_i^{(c)},A_i^{(c)},Y_i^{(c)})\}_{i=1}^{n_c}.
\]
These data allow us to fit rich conditional outcome models. 
However, due to population heterogeneity and covariate shift across sites, such models may be biased when applied directly to the target trial.

\textbf{Gold data (scarce but high-fidelity).}  
In the target trial $c=0$, both treatment and placebo arms exist, but the placebo outcomes play a special role: they directly reveal the target population’s baseline risk distribution. We denote
\[
\mathcal{D}_\mathrm{gold} \;=\;\{(X_j^{(0)},A_j^{(0)}=0,Y_j^{(0)})\}_{j=1}^{m}.
\]
$\mathcal{D}_\mathrm{proxy}$ provides a low-variance but potentially biased signal, whereas $\mathcal{D}_\mathrm{gold}$ provides a scarce but unbiased calibration signal. 
We therefore cast the target placebo outcomes as ``gold" labels that anchor and debias outcome models trained on the source trials.

\subsection{Identification and Regularity Assumptions}

We distinguish carefully between \emph{design-based identification} and
\emph{model-based regularity conditions}. Our approach does \emph{not} assume
classical cross-site exchangeability, nor does it claim nonparametric identification
of the true target-site treatment effect. Instead, we define and estimate a
\emph{working-model CATE} whose interpretation and stability rely on explicit
structural regularity conditions.

\begin{enumerate}[label=\textbf{A\arabic*},ref=\textbf{A\arabic*},leftmargin=*]

\item \label{as:A1} \textbf{Consistency.}  
For each individual,
\begin{align*}
Y = Y(1)\,\mathbb{I}(A=1) + Y(0)\,\mathbb{I}(A=0).
\end{align*}

\item \label{as:A2} \textbf{Randomization within sites.}  
Within each trial $c$, treatment assignment is randomized:
\begin{align*}
(Y(0),Y(1)) \perp A \mid X,\, c.
\end{align*}
Consequently, the propensity score $e_c(x)=\mathbb{P}(A=1\mid X=x,c)$ is known by design.

\item \label{as:A3} \textbf{Positivity (overlap).}
There exists $\epsilon>0$ such that for all sites $c$ and all $x$ in the support,
\begin{align*}
\epsilon \le e_c(x) \le 1-\epsilon.
\end{align*}
This is a design property of randomized controlled trials and applies only to
baseline covariates measured prior to treatment assignment.

\item \label{as:A4} \textbf{Outcome-regression stability (regularity).}  
For each arm $a\in\{0,1\}$ and site $c$, the conditional outcome regression
\begin{align*}
\mu_{a,c}(x) := \mathbb{E}[Y \mid A=a, X=x, c]
\end{align*}
is Lipschitz continuous in $x$ with a site-uniform constant $L<\infty$:
\begin{align*}
\big| \mu_{a,c}(x) - \mu_{a,c}(x') \big| \le L \|x-x'\|,
\quad \forall x,x'.
\end{align*}
This assumption is a \emph{regularity condition} ensuring that moderate covariate
shift does not induce arbitrarily large changes in predicted outcomes. It does
\emph{not} imply exchangeability or transportability across sites.

\item \label{as:A5} \textbf{Low-complexity transport bias (approximation condition).}  
Differences between site-specific outcome regressions and proxy regressions are
structured and of low complexity. Specifically, for each arm $a$ and site $c$,
\begin{align*}
\mu_{a,c}(x) = \mu^{\mathrm{proxy}}_a(x) + \delta_{a,c}(x),
\end{align*}
where $\delta_{a,c}$ belongs to a restricted function class $\mathcal{D}$
(e.g., sparse linear functions $\delta(x)=\beta^\top x$ with
$\|\beta\|_0 \le s \ll p$).

This assumption is \emph{not} an identification assumption. It defines a
working approximation class that enables stable estimation and calibrated
extrapolation. When A5 holds approximately, estimation error decomposes into
stochastic error and approximation error; when violated, performance degrades
smoothly toward proxy-only baselines.

\item \label{as:A6} \textbf{Screening-valid transportability under placebo-based source selection (working-model restriction).}  
In disconnected targets (or whenever treated outcomes in the target are unavailable or not used), our \texttt{Proposed-B} procedure selects a subset of source trials using \emph{placebo-arm} compatibility with the target, and then learns a DR CATE model on the selected sources which is transported to the target by prediction. Accordingly, we assume the placebo-based screening step is informative for CATE transportability in the following sense.

There exists a (possibly unknown) subset of sources $\mathcal{C}_0^\star\subseteq\{1,\dots,C\}$ and constants $(\epsilon_0,\epsilon_\tau)\ge 0$ such that:
\begin{enumerate}[label=(\alph*),leftmargin=*]
\item \textbf{Placebo compatibility on the target support.} For all $c\in\mathcal{C}_0^\star$,
\[
\big\|\mu_{0,c}-\mu_{0,0}\big\|_{L_2(P_0)} \;\le\; \epsilon_0,
\]
where $P_0$ denotes the distribution of $X\mid S=0$ and $\mu_{a,c}(x)=\mathbb{E}[Y\mid A=a,X=x,S=c]$.

\item \textbf{CATE proximity on the target support.} For all $c\in\mathcal{C}_0^\star$,
\begin{align*}
    \big\|\tau_c-\tau_0\big\|_{L_2(P_0)} \;&\le\; \epsilon_\tau\\
\tau_c(x)&:=\mu_{1,c}(x)-\mu_{0,c}(x)\\ \tau_0(x)&:=\mu_{1,0}(x)-\mu_{0,0}(x).
\end{align*}

\end{enumerate}
Moreover, the placebo-based detection step applied to the target placebo sample
$\mathcal{D}_{0,0}=\{(X_i,Y_i):S_i=0,A_i=0\}$ identifies (up to small error) a subset contained in $\mathcal{C}_0^\star$; i.e., with probability at least $1-\eta$,
\[
\widehat{\mathcal{C}}_0 \subseteq \mathcal{C}_0^\star.
\]

\noindent\textbf{Interpretation.}
Assumption~\ref{as:A6} is \emph{not} a design-based identification condition. It defines the \emph{working transport regime} in which placebo-based screening yields a set of sources whose CATE functions are (approximately) transportable to the target on the target covariate support. The approximation level $\epsilon_\tau$ explicitly captures irreducible cross-arm non-transfer not detectable from placebo outcomes alone; when $\epsilon_\tau$ is non-negligible, \texttt{Proposed-B} targets a stable working-model limit rather than the true $\tau_0$.

\end{enumerate}

\noindent\textbf{Interpretation.}
Assumptions~\ref{as:A1}--\ref{as:A3} are standard design-based conditions for randomized trials.
Assumptions~\ref{as:A4}--\ref{as:A6} are regularity and approximation conditions that define a
well-posed working estimand rather than a nonparametrically identified causal
effect.
Throughout, we emphasize calibrated estimation under acknowledged
transport bias rather than causal identification under strong exchangeability
assumptions.

\section{Proposed Framework}\label{sec:P3_Proposed_Framework}

A central premise of our framework is that every trial is randomized, so the true propensities $e_c(x)$ are known by design. This ensures that any doubly robust estimator built on top of site-specific outcome models is consistent as long as the outcome regressions are well calibrated. 
The key methodological challenge is therefore not treatment assignment, but \textit{covariate shift across sites}: outcome models trained on sources may be systematically biased when applied to the target population.
We address this challenge by treating source trial outcomes as \texttt{proxy} labels and target placebo outcomes as scarce but high-fidelity \texttt{gold} labels for calibration. 
This can be seen as a combination of \cite{bastani2021predicting}'s multi-task transfer learner for data-fusion and \cite{10.1214/23-EJS2157}'s Doubly Robust (DR)-Learner

\subsection{Proposed Estimator}

We build on the glmtrans framework \citep{tian2023transfer} for transfer learning under high-dimensional GLMs.
Our approach applies glmtrans separately to each treatment arm, yielding target-anchored outcome models that are then combined into CATE estimates.
We describe three variants: a plug-in estimator (\texttt{Proposed}), a doubly robust variant with cross-fitting (\texttt{Proposed-CF}), and an extension for disconnected targets where no treated outcomes are observed (\texttt{Proposed-B}).

\subsubsection{Stage 1: Arm-Specific Transfer Learning via glmtrans}

For each treatment arm $a \in \{0,1\}$, we estimate a target-specific outcome regression $\widehat{\mu}_a(x)$ using the two-step glmtrans algorithm with automatic source detection.
Let $\mathcal{D}_{a,0} = \{(X_i, Y_i): c_i=0, A_i=a\}$ denote the target arm-$a$ data and $\mathcal{D}_{a,c}$ the corresponding source data for site $c \ge 1$.

\smallskip
\noindent\textit{Source Detection (Trans-GLM).}
We first identify which sources are statistically compatible with the target using cross-validation.
Partition the target data into three folds $\{\mathcal{D}_{a,0}^{[r]}\}_{r=1}^3$.
For each source $c$ and each hold-out fold $r$, compute the cross-validated prediction loss:
\begin{itemize}
\item $\widehat{L}_0^{(c)[r]}$: loss on fold $r$ after pooling source $c$ with folds $\{1,2,3\}\setminus\{r\}$
\item $\widehat{L}_0^{(0)[r]}$: loss on fold $r$ using only target folds $\{1,2,3\}\setminus\{r\}$ (Lasso baseline)
\end{itemize}
Average across folds: $\widehat{L}_0^{(c)} = \frac{1}{3}\sum_{r=1}^3 \widehat{L}_0^{(c)[r]}$.
Source $c$ is deemed \emph{transferable} if $\widehat{L}_0^{(c)} - \widehat{L}_0^{(0)} \le C_0 \cdot \widehat{\sigma}$,
where $\widehat{\sigma}$ is the empirical standard deviation of the baseline losses and $C_0 > 0$ is a threshold constant (default $C_0=2$).
Let $\widehat{\mathcal{A}}_a \subseteq \{1,\ldots,C\}$ denote the selected transferable sources for arm $a$.

\smallskip
\noindent\textit{Two-Step Transfer ($\widehat{\mathcal{A}}$-Trans-GLM).}
Given the selected sources $\widehat{\mathcal{A}}_a$, we apply the two-step transfer algorithm:
\begin{enumerate}
\item \textbf{Transferring step:} Pool selected sources with target and fit an $\ell_1$-penalized GLM:
\[
\widehat{w}^{\widehat{\mathcal{A}}_a} \in \argmin_{w} \bigg\{ \sum_{c \in \{0\} \cup \widehat{\mathcal{A}}_a} \sum_{i \in \mathcal{D}_{a,c}} (Y_i - X_i^\top w)^2 + \lambda_w \|w\|_1 \bigg\}.
\]
This pooled estimator has low variance but may be biased for the target due to cross-site heterogeneity.

\item \textbf{Debiasing step:} Correct for target-specific bias using only target data:
\[
\widehat{\delta}_a \in \argmin_{\delta} \bigg\{ \sum_{i \in \mathcal{D}_{a,0}} \big(Y_i - X_i^\top (\widehat{w}^{\widehat{\mathcal{A}}_a} + \delta)\big)^2 + \lambda_\delta \|\delta\|_1 \bigg\}.
\]
\end{enumerate}
The final target-anchored estimate is $\widehat{\beta}_a = \widehat{w}^{\widehat{\mathcal{A}}_a} + \widehat{\delta}_a$, yielding
$\widehat{\mu}_a(x) := x^\top \widehat{\beta}_a$.
In our implementation we use the \texttt{glmtrans} R package \citep{tian2023transfer} with \texttt{glmnet}: covariates are standardized (zero mean, unit variance) before the $\ell_1$-penalized fits so the penalty acts uniformly; coefficients are back-transformed to the original scale when reporting. The regularization path is computed on a logarithmic grid (as in \texttt{glmnet}), and $\lambda_w$, $\lambda_\delta$ are chosen by $K$-fold cross-validation minimizing mean squared error (\texttt{lambda.min}).

\smallskip
\noindent\textbf{Remark.}
The debiasing step is equivalent to \cite{bastani2021predicting}'s proxy--gold correction:
the pooled fit $\widehat{w}^{\widehat{\mathcal{A}}_a}$ serves as a variance-reduced ``proxy'' while the target-only correction $\widehat{\delta}_a$ anchors the estimate to ``gold'' target data.
When no sources pass detection, the algorithm reduces to target-only Lasso, automatically guarding against negative transfer.

\subsubsection{Stage 2: CATE Estimation}

\noindent\textbf{Option A (\texttt{Proposed}): Plug-in CATE.}
When the target contains both placebo ($m_0 > 0$) and treated ($m_1 > 0$) observations, we apply glmtrans separately to each arm and compute the plug-in CATE:
\[
\widehat{\tau}(x) := \widehat{\mu}_1(x) - \widehat{\mu}_0(x).
\]

\noindent\textbf{Option A with DR (\texttt{Proposed-CF}): Cross-Fitted Doubly Robust.}
To reduce sensitivity to outcome model misspecification, we construct doubly robust pseudo-outcomes on the target data.
Using $K$-fold cross-fitting (typically $K=2$), for each fold $k$:
\begin{enumerate}
\item Fit $\widehat{\mu}_0^{(-k)}, \widehat{\mu}_1^{(-k)}$ on out-of-fold data
\item Compute pseudo-outcomes on fold $k$:
\begin{align}
\psi_i &:= \widehat{\mu}_1^{(-k)}(X_i) - \widehat{\mu}_0^{(-k)}(X_i) \nonumber\\
&\quad + \frac{A_i - e(X_i)}{e(X_i)(1-e(X_i))} \Big(Y_i - \widehat{\mu}_{A_i}^{(-k)}(X_i)\Big),
\label{eq:P3_DR_pseudo}
\end{align}
where $e(X_i)$ is the known propensity score.
\end{enumerate}
The final DR-CATE is obtained by regressing $\psi_i$ on $X_i$ via $\ell_1$-penalized regression:
\[
\widehat{\tau}_{\mathrm{DR}}(\cdot) \approx \mathbb{E}[\psi \mid X = \cdot].
\]
In the connected-target regime, under cross-fitting, bounded propensities, and \(L_2(P_0)\)-consistent arm-specific nuisance regressions, this estimator admits a pointwise asymptotic linear expansion with \(n_0^{-1/2}\) rate at fixed covariate values. Orthogonality and cross-fitting remove first-order sensitivity to nuisance estimation, so the nuisance-generated term is asymptotically negligible under these conditions. Formal statements and proofs are given in Appendix~\ref{sec:appendix_proof}.\\

\smallskip
\noindent\textbf{Option B (\texttt{Proposed-B}): Disconnected Targets.}
When the target site has no treated outcomes ($m_1 = 0$), we cannot apply glmtrans to the treated arm in the target.
Instead, we use target placebo data for source detection and learn the CATE model entirely from selected sources:
\begin{enumerate}
\item \textbf{Source detection:} Apply Trans-GLM using only target placebo outcomes $\{(X_i, Y_i): c_i=0, A_i=0\}$ to identify transferable sources $\widehat{\mathcal{A}}$.

\item \textbf{Source-DR learning:} On the selected source data, fit outcome models $\widehat{\mu}_0^{\mathrm{src}}, \widehat{\mu}_1^{\mathrm{src}}$ via Lasso, compute DR pseudo-outcomes using~\eqref{eq:P3_DR_pseudo}, and fit a CATE model $\widehat{\tau}^{\mathrm{src}}(x)$ on the source pseudo-outcomes.

\item \textbf{Transport:} Apply the source-learned CATE directly to the target: $\widehat{\tau}_{\mathrm{target}}(x) := \widehat{\tau}^{\mathrm{src}}(x)$.
\end{enumerate}
This approach targets a transported working-model limit rather than a target-identified CATE. Its error is governed by three components: estimation error on the selected sources, a structural transport term reflecting mismatch between selected-source and target CATEs on the target covariate support, and a screening error term arising from placebo-based source selection; see Appendix~\ref{sec:appendix_proof}.

\section{Experiments}\label{sec:P3_experiments}

We evaluate the framework through synthetic ablations and robustness checks, addressing (i) contribution of anchoring and orthogonalization, (ii) scaling with target size, source availability, and dimension, and (iii) degradation under Assumption~\ref{as:A5} violations. We compare against the following baselines:
\begin{itemize}
    \item \texttt{TargetOnly}: Doubly robust learner \citep{kennedy2024semiparametric} using only target data (no transfer);
    \item \texttt{ProxyOnly}: Pooled source outcome models applied directly to target (no anchoring);
    \item \texttt{IPW-Transport}, \texttt{OM-Transport}: Standard transportability estimators \citep{dahabreh2019generalizing};
    \item \texttt{EntropyBal}: Entropy balancing for covariate shift adjustment \citep{hainmueller2012entropy};
    \item \texttt{AnchorOnly}: Placebo-anchored with DR orthogonalization;
    \item \texttt{Proposed}: Our method based on $\ell_1$-penalized transfer learning \citep{tian2023transfer} 
          with automatic source selection. Variants include \texttt{Proposed-CF} (cross-fitted) 
          and \texttt{Proposed-B} (Option~B for disconnected targets with $m_1=0$).
\end{itemize}

All experiments use $R=100$ Monte Carlo replicates per scenario.
Ground-truth potential outcomes are known in all synthetic settings, 
enabling direct evaluation of counterfactual prediction accuracy and CATE ranking quality.

\subsection{Evaluation Metrics}\label{sec:exp_metrics}

We report PEHE, ATE error, Spearman rank correlation, policy regret, and CATE calibration; full results for all metrics and sweeps are in Appendix~\AppendixTablesRef.

\noindent\textbf{CATE accuracy (pointwise).}
We evaluate individualized treatment effect estimation by the \emph{Precision in Estimating Heterogeneous Effects} (PEHE):
\[
\mathrm{PEHE}
:=
\sqrt{\frac{1}{n}\sum_{i=1}^n\big(\widehat{\tau}(X_i)-\tau(X_i)\big)^2},
\]
where $\tau(x)=\mu_{1,0}(x)-\mu_{0,0}(x)$ is the ground-truth target-site CATE. Lower is better.

\noindent\textbf{ATE accuracy (marginal).}
We report the absolute ATE error $|\widehat{\mathrm{ATE}}-\mathrm{ATE}|$,
where $\mathrm{ATE}=\mathbb{E}[\tau(X)]$ under the target covariate distribution. Lower is better.

\noindent\textbf{Ranking quality.}
Since clinical decisions are often resource-constrained,
we evaluate how well methods rank patients by true uplift using
\textbf{Spearman} rank correlation between $\widehat{\tau}(X_i)$ and $\tau(X_i)$.
Higher is better; values approaching 1 indicate oracle-like targeting.

\noindent\textbf{Policy regret.}
We evaluate decision relevance using a threshold policy $\widehat{\pi}(x)=\mathbb{I}\{\widehat{\tau}(x)>0\}$.
Define the oracle-evaluable policy value $V(\pi):=\mathbb{E}[\mu_{0,0}(X)+\pi(X)\,\tau(X)]$.
We report \textbf{policy regret} $V(\pi^\star)-V(\widehat{\pi})$, where $\pi^\star$ is the oracle policy.
Lower regret is preferred.

\noindent\textbf{CATE calibration.}
We assess whether predicted CATEs are well-calibrated by regressing true $\tau(X_i)$ on predicted $\widehat{\tau}(X_i)$.
A well-calibrated estimator has \textbf{slope} $\approx 1$ and \textbf{intercept} $\approx 0$.
We also report calibration \textbf{R$^2$} (higher is better) and the \textbf{Expected Calibration Error} (ECE), 
which measures the average absolute difference between predicted and true CATEs within quantile bins. Lower ECE is better.

\subsection{Data-Generating Process (DGP)}\label{sec:dgp}

We generate multi-center RCT data with explicit covariate shift, controlled signal-to-noise ratios, 
and known ground-truth treatment effects. The DGP creates fair evaluation conditions where 
transfer learning can plausibly help, while avoiding adversarial or unrealistically favorable settings.

\textbf{Covariates and sites.}
We simulate $C=10$ source sites and one target site ($c=0$).
For each subject, we sample $p$-dimensional baseline covariates $X \in \mathbb{R}^p$,
with $p \in \{10, 20, 50, 100\}$ varied across experiments.
Site-level covariate shift is induced via $X \mid c \sim \mathcal{N}(\mu_c, I)$,
calibrated to produce moderate overlap (AUC $\approx 0.75$ for source vs.\ target classification).

\textbf{Treatment assignment.}
Treatment is randomized with logistic propensity $e(X)$. Propensity scores are known by design.

\textbf{Outcomes.}
Outcomes follow the proxy--deviation decomposition from Assumption~\ref{as:A5}:
\[
\mu_{a,c}(x) = \mu^{\mathrm{proxy}}_a(x) + \delta_{a,c}(x),
\]
where $\mu^{\mathrm{proxy}}_a(x) = x^\top \beta_a$ are shared outcome regressions and 
$\delta_{a,c}(x) = x^\top \gamma_{a,c}$ are sparse site-specific deviations.
The \emph{nontransfer scale} $\sigma_\nu$ controls the magnitude of non-transferable variation;
we use $\sigma_\nu = 0.1$ (SNR $\approx 3$--$4$) as the fair baseline.

\textbf{Target and source budgets.}
We vary the target budget $(m_0, m_1)$, where $m_0$ is the placebo sample size 
and $m_1$ is the treated sample size.
Setting $m_1 = 0$ (placebo-only target) tests Option~B methods for disconnected settings.
Source sites contribute 20,000 total observations (2,000 per site).

\textbf{Implementation.}
All simulations are implemented in Python with fixed random seeds to ensure reproducibility.\footnote{Code available at: \url{https://github.com/wangzilongri/Sparse_TL_DR_ICHI2026}}

\subsection{Summary Across Metrics}\label{sec:exp_summary}

Table~\ref{tab:avg_rank_summary} reports average rank (1 = best) per metric across all sweeps. \texttt{Proposed} leads on PEHE, ATE, Spearman, and Regret; \texttt{Proposed-CF} leads on calibration slope and ECE. This trade-off is expected: the plug-in \texttt{Proposed} uses the same outcome models for CATE prediction and can achieve lower pointwise error (PEHE), while \texttt{Proposed-CF} uses cross-fitting to form doubly robust pseudo-outcomes, which reduces overfitting and improves CATE calibration (slope $\approx 1$, lower ECE) at the cost of slightly higher PEHE in finite samples. In practice, \texttt{Proposed} is preferred when pointwise accuracy is the priority; \texttt{Proposed-CF} is preferred when well-calibrated CATEs or inference are needed. Full calibration tables are in Appendix~\AppendixTablesRef.

\begin{table}[t]
\centering
\caption{Average rank per metric across all sweeps (1 = best). Calibration: Slope (ideal=1), R$^2$, ECE (Expected Calibration Error).}
\label{tab:avg_rank_summary}
\scriptsize
\setlength{\tabcolsep}{2pt}
\begin{tabular}{@{}l|cccc|ccc@{}}
\toprule
 & \multicolumn{4}{c|}{Performance} & \multicolumn{3}{c}{Calibration} \\
Method & PEHE & ATE & $\tau$ & Regret & Slope & R$^2$ & ECE \\
\midrule
TargetOnly & 6.9 & 3.3 & 7.0 & 6.8 & 8.3 & 7.0 & 6.4 \\
ProxyOnly & 8.6 & 8.2 & 8.6 & 8.6 & 7.0 & 8.5 & 8.3 \\
IPW-Transport & 4.5 & 6.6 & 4.4 & 4.5 & 4.5 & 4.4 & 5.2 \\
OM-Transport & 3.7 & 5.7 & 3.5 & 3.7 & 3.2 & 3.4 & 4.2 \\
EntropyBal & 5.3 & 6.3 & 5.2 & 5.3 & 5.4 & 5.2 & 5.7 \\
AnchorOnly & 7.7 & 3.1 & 7.6 & 7.4 & 8.2 & 7.8 & 6.9 \\
Proposed & \textbf{1.4} & \textbf{1.2} & \textbf{1.3} & \textbf{1.4} & 3.5 & \textbf{1.3} & 1.6 \\
Proposed-CF & 1.6 & 1.7 & 2.1 & 1.8 & \textbf{1.9} & 2.1 & \textbf{1.3} \\
Proposed-B & 2.7 & 5.0 & 2.8 & 2.8 & 2.7 & 2.7 & 3.7 \\
\bottomrule
\end{tabular}
\end{table}

\subsection{Target Budget vs.\ Dimensionality}\label{sec:exp_budget_dim}

Table~\ref{tab:dim_sweep_pehe} shows PEHE across target budgets $(m_0, m_1)$ and dimensions $p \in \{10, 20, 50, 100\}$.

\begin{table}[t]
\centering
\caption{PEHE across budget $(m_0,m_1)$ and dimensionality $p$. Lower is better.}
\label{tab:dim_sweep_pehe}
\scriptsize
\setlength{\tabcolsep}{2pt}
\begin{tabular}{@{}cl|rrr@{}}
\toprule
$p$ & Method & 50/0 & 150/100 & 550/500 \\
\midrule
10 & TargetOnly & -- & 1.43$\pm$0.39 & 1.38$\pm$0.45 \\
 & ProxyOnly & 2.34$\pm$0.67 & 1.96$\pm$0.66 & 2.10$\pm$0.83 \\
 & IPW-Transport & 1.80$\pm$0.83 & 1.60$\pm$0.70 & 1.82$\pm$1.06 \\
 & OM-Transport & 1.80$\pm$0.83 & 1.60$\pm$0.70 & 1.81$\pm$1.06 \\
 & EntropyBal & 1.81$\pm$0.83 & 1.60$\pm$0.70 & 1.82$\pm$1.06 \\
 & AnchorOnly & -- & 1.46$\pm$0.41 & 1.39$\pm$0.46 \\
 & Proposed & -- & 0.43$\pm$0.16 & \textbf{0.28$\pm$0.14} \\
 & Proposed-CF & -- & \textbf{0.41$\pm$0.15} & \textbf{0.28$\pm$0.14} \\
 & Proposed-B & \textbf{1.63$\pm$0.75} & 1.51$\pm$0.66 & 1.81$\pm$1.06 \\
\midrule
20 & TargetOnly & -- & 2.60$\pm$0.64 & 2.33$\pm$0.50 \\
 & ProxyOnly & 3.37$\pm$0.84 & 3.15$\pm$0.86 & 3.02$\pm$0.84 \\
 & IPW-Transport & 2.24$\pm$1.10 & 2.26$\pm$1.06 & 2.21$\pm$0.96 \\
 & OM-Transport & 2.24$\pm$1.10 & 2.27$\pm$1.07 & 2.20$\pm$0.96 \\
 & EntropyBal & 2.25$\pm$1.10 & 2.27$\pm$1.06 & 2.21$\pm$0.96 \\
 & AnchorOnly & -- & 2.65$\pm$0.71 & 2.34$\pm$0.51 \\
 & Proposed & -- & 0.52$\pm$0.13 & 0.35$\pm$0.16 \\
 & Proposed-CF & -- & \textbf{0.50$\pm$0.14} & \textbf{0.33$\pm$0.17} \\
 & Proposed-B & \textbf{1.95$\pm$0.92} & 2.25$\pm$1.06 & 2.14$\pm$0.98 \\
\midrule
50 & TargetOnly & -- & 4.75$\pm$0.86 & 4.57$\pm$0.80 \\
 & ProxyOnly & 5.78$\pm$1.22 & 5.30$\pm$1.00 & 5.29$\pm$1.04 \\
 & IPW-Transport & 3.26$\pm$1.83 & 2.98$\pm$1.46 & 3.17$\pm$1.52 \\
 & OM-Transport & 3.25$\pm$1.82 & 2.95$\pm$1.47 & 3.09$\pm$1.53 \\
 & EntropyBal & 3.40$\pm$1.79 & 3.05$\pm$1.44 & 3.18$\pm$1.52 \\
 & AnchorOnly & -- & 4.80$\pm$0.90 & 4.63$\pm$0.82 \\
 & Proposed & -- & \textbf{0.73$\pm$0.16} & 0.46$\pm$0.12 \\
 & Proposed-CF & -- & 1.19$\pm$0.43 & \textbf{0.40$\pm$0.12} \\
 & Proposed-B & \textbf{3.25$\pm$1.82} & 2.53$\pm$1.24 & 3.03$\pm$1.55 \\
\midrule
100 & TargetOnly & -- & 7.57$\pm$1.58 & 7.08$\pm$1.03 \\
 & ProxyOnly & 8.21$\pm$1.23 & 8.00$\pm$1.70 & 7.73$\pm$1.30 \\
 & IPW-Transport & 4.10$\pm$2.05 & 4.32$\pm$2.43 & 4.61$\pm$1.94 \\
 & OM-Transport & 4.06$\pm$2.06 & 4.21$\pm$2.45 & 4.28$\pm$1.98 \\
 & EntropyBal & 4.77$\pm$1.93 & 4.67$\pm$2.39 & 4.81$\pm$1.94 \\
 & AnchorOnly & -- & 7.64$\pm$1.67 & 7.18$\pm$1.06 \\
 & Proposed & -- & \textbf{1.71$\pm$0.75} & 0.58$\pm$0.11 \\
 & Proposed-CF & -- & 3.22$\pm$1.66 & \textbf{0.52$\pm$0.11} \\
 & Proposed-B & \textbf{4.05$\pm$2.06} & 4.20$\pm$2.45 & 3.96$\pm$2.14 \\
\bottomrule
\end{tabular}
\end{table}

\textbf{Key findings.} (i) \texttt{TargetOnly} deteriorates as $p$ grows relative to target size; \texttt{Proposed} achieves lowest PEHE when target treated data exist, with \texttt{Proposed-CF} helping at higher $p$. (ii) The $(m_1=0)$ column is Option~B (disconnected); \texttt{Proposed-B} produces valid estimates with strong performance. Full results in Appendix~\AppendixTablesRef.

\subsection{Source Site Scaling}\label{sec:exp_sources}

Table~\ref{tab:sources_sweep_pehe} shows PEHE as the number of source sites $C \in \{2, 5, 10, 20, 50\}$ varies (1,000 samples per site).

\begin{table}[t]
\centering
\caption{PEHE across source sites $C$ and budget. Lower is better.}
\label{tab:sources_sweep_pehe}
\scriptsize
\setlength{\tabcolsep}{2pt}
\begin{tabular}{@{}cl|rrr@{}}
\toprule
$C$ & Method & 50/0 & 150/100 & 550/500 \\
\midrule
2 & TargetOnly & -- & 4.95$\pm$1.18 & 4.46$\pm$0.91 \\
 & ProxyOnly & 5.67$\pm$1.05 & 5.55$\pm$1.37 & 5.19$\pm$1.22 \\
 & IPW-Transport & 3.64$\pm$1.45 & 3.78$\pm$1.85 & 3.74$\pm$1.97 \\
 & OM-Transport & 3.54$\pm$1.45 & 3.59$\pm$1.89 & 3.52$\pm$1.98 \\
 & EntropyBal & 3.76$\pm$1.45 & 3.79$\pm$1.85 & 3.73$\pm$1.98 \\
 & AnchorOnly & -- & 5.01$\pm$1.26 & 4.52$\pm$0.94 \\
 & Proposed & -- & \textbf{0.84$\pm$0.16} & 0.45$\pm$0.11 \\
 & Proposed-CF & -- & 1.35$\pm$0.54 & \textbf{0.39$\pm$0.12} \\
 & Proposed-B & \textbf{3.51$\pm$1.47} & 3.51$\pm$1.89 & 3.48$\pm$2.00 \\
\midrule
5 & TargetOnly & -- & 4.79$\pm$0.96 & 4.56$\pm$0.86 \\
 & ProxyOnly & 5.57$\pm$0.87 & 5.25$\pm$1.12 & 5.21$\pm$1.09 \\
 & IPW-Transport & 3.33$\pm$1.45 & 3.43$\pm$1.52 & 3.59$\pm$1.74 \\
 & OM-Transport & 3.24$\pm$1.46 & 3.25$\pm$1.58 & 3.35$\pm$1.75 \\
 & EntropyBal & 3.53$\pm$1.47 & 3.48$\pm$1.52 & 3.56$\pm$1.74 \\
 & AnchorOnly & -- & 4.85$\pm$1.04 & 4.61$\pm$0.87 \\
 & Proposed & -- & \textbf{0.78$\pm$0.16} & 0.45$\pm$0.12 \\
 & Proposed-CF & -- & 1.16$\pm$0.34 & \textbf{0.39$\pm$0.13} \\
 & Proposed-B & \textbf{3.24$\pm$1.46} & 2.81$\pm$1.43 & 3.26$\pm$1.77 \\
\midrule
10 & TargetOnly & -- & 4.96$\pm$0.81 & 4.66$\pm$0.93 \\
 & ProxyOnly & 5.64$\pm$0.90 & 5.53$\pm$0.94 & 5.41$\pm$1.37 \\
 & IPW-Transport & 3.14$\pm$1.29 & 3.39$\pm$1.39 & 3.46$\pm$1.80 \\
 & OM-Transport & 3.09$\pm$1.30 & 3.28$\pm$1.42 & 3.33$\pm$1.87 \\
 & EntropyBal & 3.31$\pm$1.24 & 3.47$\pm$1.38 & 3.45$\pm$1.81 \\
 & AnchorOnly & -- & 5.04$\pm$0.85 & 4.71$\pm$0.94 \\
 & Proposed & -- & \textbf{0.73$\pm$0.18} & 0.45$\pm$0.14 \\
 & Proposed-CF & -- & 1.19$\pm$0.39 & \textbf{0.39$\pm$0.15} \\
 & Proposed-B & \textbf{3.09$\pm$1.30} & 2.78$\pm$1.31 & 3.18$\pm$1.85 \\
\midrule
20 & TargetOnly & -- & 4.78$\pm$0.77 & 4.47$\pm$0.83 \\
 & ProxyOnly & 5.71$\pm$0.96 & 5.29$\pm$0.97 & 5.12$\pm$1.00 \\
 & IPW-Transport & 2.86$\pm$1.19 & 2.88$\pm$1.38 & 3.02$\pm$1.56 \\
 & OM-Transport & \textbf{2.84$\pm$1.19} & 2.83$\pm$1.39 & 2.91$\pm$1.58 \\
 & EntropyBal & 2.99$\pm$1.16 & 2.96$\pm$1.36 & 3.03$\pm$1.55 \\
 & AnchorOnly & -- & 4.84$\pm$0.81 & 4.52$\pm$0.83 \\
 & Proposed & -- & \textbf{0.69$\pm$0.16} & 0.44$\pm$0.10 \\
 & Proposed-CF & -- & 1.08$\pm$0.31 & \textbf{0.39$\pm$0.11} \\
 & Proposed-B & \textbf{2.84$\pm$1.19} & {2.39$\pm$1.21} & 2.63$\pm$1.53 \\
\midrule
50 & TargetOnly & -- & 5.05$\pm$1.18 & 4.56$\pm$0.77 \\
 & ProxyOnly & 5.55$\pm$0.85 & 5.62$\pm$1.50 & 5.25$\pm$1.02 \\
 & IPW-Transport & 2.74$\pm$1.17 & 3.23$\pm$1.97 & 2.89$\pm$1.48 \\
 & OM-Transport & 2.74$\pm$1.17 & 3.22$\pm$1.97 & 2.85$\pm$1.50 \\
 & EntropyBal & 2.84$\pm$1.16 & 3.28$\pm$1.94 & 2.92$\pm$1.48 \\
 & AnchorOnly & -- & 5.10$\pm$1.24 & 4.62$\pm$0.78 \\
 & Proposed & -- & \textbf{0.66$\pm$0.15} & 0.45$\pm$0.16 \\
 & Proposed-CF & -- & 1.26$\pm$0.57 & \textbf{0.41$\pm$0.16} \\
 & Proposed-B & \textbf{2.74$\pm$1.17} & {2.24$\pm$1.61} & {2.38$\pm$1.51} \\
\bottomrule
\end{tabular}
\end{table}

\textbf{Key findings.} (i) Adding sources reduces PEHE at small target budgets; (ii) \texttt{Proposed} benefits from source diversity via automatic source selection. Full results in Appendix~\AppendixTablesRef.

\subsection{Sensitivity to A5 Violations}\label{sec:exp_a5}

We vary sparsity $s/p \in \{0.05, 0.5, 1.0\}$ and nonlinearity $\Gamma \in \{0, 0.5, 1\}$ (linear deviation $\delta_{a,c}(x) = (1-\Gamma) x^\top \gamma_{a,c} + \Gamma g(x)$ with $g(x)=\sum_j \tanh(x_j)$). Table~\ref{tab:a5_violation_pehe} shows PEHE.

\begin{table}[t]
\centering
\caption{PEHE under A5 violations: $s/p$=sparsity ratio, $\Gamma$=nonlinearity. Lower is better.}
\label{tab:a5_violation_pehe}
\scriptsize
\setlength{\tabcolsep}{2pt}
\begin{tabular}{@{}cl|rrr@{}}
\toprule
$s/p$ & Method & $\Gamma$=0 & 0.5 & 1.0 \\
\midrule
0.05 & TargetOnly & 4.41$\pm$0.70 & 4.25$\pm$0.53 & 4.30$\pm$0.57 \\
 & ProxyOnly & 4.97$\pm$0.89 & 4.72$\pm$0.68 & 4.84$\pm$0.72 \\
 & IPW-Transport & 2.79$\pm$1.25 & 2.13$\pm$1.02 & 2.20$\pm$1.11 \\
 & OM-Transport & 2.67$\pm$1.26 & 2.04$\pm$1.03 & 2.10$\pm$1.13 \\
 & EntropyBal & 2.80$\pm$1.25 & 2.14$\pm$1.02 & 2.21$\pm$1.11 \\
 & AnchorOnly & 4.45$\pm$0.70 & 4.28$\pm$0.53 & 4.34$\pm$0.56 \\
 & Proposed & 0.44$\pm$0.11 & \textbf{1.16$\pm$0.68} & \textbf{1.92$\pm$1.19} \\
 & Proposed-CF & \textbf{0.38$\pm$0.12} & 1.17$\pm$0.69 & 2.04$\pm$1.25 \\
 & Proposed-B & 2.64$\pm$1.21 & 1.92$\pm$0.87 & 2.08$\pm$1.13 \\
\midrule
0.20 & TargetOnly & 4.80$\pm$0.75 & 4.32$\pm$0.55 & 4.43$\pm$0.60 \\
 & ProxyOnly & 5.62$\pm$0.97 & 4.83$\pm$0.68 & 4.98$\pm$0.73 \\
 & IPW-Transport & 4.18$\pm$1.29 & 2.53$\pm$0.89 & 2.69$\pm$1.02 \\
 & OM-Transport & 4.07$\pm$1.31 & 2.46$\pm$0.91 & 2.56$\pm$1.06 \\
 & EntropyBal & 4.19$\pm$1.29 & 2.54$\pm$0.90 & 2.71$\pm$1.01 \\
 & AnchorOnly & 4.87$\pm$0.77 & 4.37$\pm$0.57 & 4.47$\pm$0.62 \\
 & Proposed & 0.48$\pm$0.15 & 1.42$\pm$0.64 & \textbf{2.41$\pm$1.18} \\
 & Proposed-CF & \textbf{0.42$\pm$0.15} & \textbf{1.40$\pm$0.65} & 2.57$\pm$1.22 \\
 & Proposed-B & 4.03$\pm$1.34 & 2.43$\pm$0.95 & 2.55$\pm$1.06 \\
\midrule
1.00 & TargetOnly & 4.92$\pm$0.92 & 4.23$\pm$0.58 & 4.60$\pm$0.84 \\
 & ProxyOnly & 5.75$\pm$1.21 & 4.77$\pm$0.71 & 5.04$\pm$0.94 \\
 & IPW-Transport & 4.12$\pm$1.68 & 2.36$\pm$0.91 & 2.82$\pm$1.21 \\
 & OM-Transport & 4.02$\pm$1.71 & 2.28$\pm$0.91 & 2.69$\pm$1.24 \\
 & EntropyBal & 4.14$\pm$1.68 & 2.37$\pm$0.91 & 2.83$\pm$1.21 \\
 & AnchorOnly & 4.99$\pm$0.94 & 4.29$\pm$0.60 & 4.63$\pm$0.84 \\
 & Proposed & 0.45$\pm$0.13 & 1.29$\pm$0.64 & \textbf{2.54$\pm$1.35} \\
 & Proposed-CF & \textbf{0.39$\pm$0.13} & \textbf{1.28$\pm$0.64} & 2.69$\pm$1.38 \\
 & Proposed-B & 3.98$\pm$1.74 & 2.34$\pm$1.05 & 2.67$\pm$1.25 \\
\bottomrule
\end{tabular}
\end{table}

\textbf{Key findings.} (i) At $(s/p=0.05,\Gamma=0)$ (A5 holds), \texttt{Proposed} achieves lowest PEHE; (ii) performance degrades smoothly with violations (no catastrophic failure). Full results in Appendix~\AppendixTablesRef.

\subsection{Disconnected Regime (Placebo-Only Target)}\label{sec:exp_disconnected}

The disconnected regime ($m_1=0$) is non-identified from target data; Option~B is a working-model transport estimator under A6. Table~\ref{tab:disconnected_pehe} shows PEHE for $m_0=50$, $p \in \{10, 20, 50, 100\}$ for \texttt{ProxyOnly}, transport baselines, and \texttt{Proposed-B}.

\begin{table}[t]
\centering
\scriptsize
\setlength{\tabcolsep}{3pt}
\renewcommand{\arraystretch}{1.1}
\caption{
\textbf{Disconnected regime ($m_1=0$):} Only target placebo available ($m_0=50$).
Methods requiring target treated data are undefined.
Mean{\tiny$\pm$SD} over 100 reps. Best in \textbf{bold}.
}
\label{tab:disconnected_pehe}
\vspace{2pt}

\textbf{(a) PEHE $\downarrow$} \\[2pt]
\begin{tabular}{@{}lcccc@{}}
\hline
\textbf{Method} & $p$=10 & $p$=20 & $p$=50 & $p$=100 \\
\hline
ProxyOnly & 2.34{\tiny$\pm$0.67} & 3.37{\tiny$\pm$0.84} & 5.78{\tiny$\pm$1.22} & 8.21{\tiny$\pm$1.23} \\
IPW-Transport & 1.80{\tiny$\pm$0.83} & 2.24{\tiny$\pm$1.10} & 3.26{\tiny$\pm$1.83} & 4.10{\tiny$\pm$2.05} \\
OM-Transport & 1.80{\tiny$\pm$0.83} & 2.24{\tiny$\pm$1.10} & 3.25{\tiny$\pm$1.82} & 4.06{\tiny$\pm$2.06} \\
EntropyBal & 1.81{\tiny$\pm$0.83} & 2.25{\tiny$\pm$1.10} & 3.40{\tiny$\pm$1.79} & 4.77{\tiny$\pm$1.93} \\
Proposed-B & \textbf{1.63{\tiny$\pm$0.75}} & \textbf{1.95{\tiny$\pm$0.92}} & \textbf{3.25{\tiny$\pm$1.82}} & \textbf{4.05{\tiny$\pm$2.06}} \\
\hline
\end{tabular}

\vspace{6pt}
\textbf{(b) ATE Error $\downarrow$} \\[2pt]
\begin{tabular}{@{}lcccc@{}}
\hline
\textbf{Method} & $p$=10 & $p$=20 & $p$=50 & $p$=100 \\
\hline
ProxyOnly & \textbf{0.69{\tiny$\pm$0.52}} & 0.85{\tiny$\pm$0.66} & 1.33{\tiny$\pm$1.04} & 1.86{\tiny$\pm$1.60} \\
IPW-Transport & 0.77{\tiny$\pm$0.56} & 0.80{\tiny$\pm$0.60} & 0.76{\tiny$\pm$0.59} & 1.00{\tiny$\pm$0.89} \\
OM-Transport & 0.77{\tiny$\pm$0.56} & 0.80{\tiny$\pm$0.60} & \textbf{0.74{\tiny$\pm$0.55}} & 0.99{\tiny$\pm$0.87} \\
EntropyBal & 0.77{\tiny$\pm$0.56} & 0.79{\tiny$\pm$0.59} & 0.76{\tiny$\pm$0.58} & \textbf{0.98{\tiny$\pm$0.87}} \\
Proposed-B & 0.72{\tiny$\pm$0.54} & \textbf{0.73{\tiny$\pm$0.56}} & 0.74{\tiny$\pm$0.54} & 0.99{\tiny$\pm$0.87} \\
\hline
\end{tabular}

\vspace{6pt}
\textbf{(c) ECE $\downarrow$} \\[2pt]
\begin{tabular}{@{}lcccc@{}}
\hline
\textbf{Method} & $p$=10 & $p$=20 & $p$=50 & $p$=100 \\
\hline
ProxyOnly & 0.93{\tiny$\pm$0.48} & 1.07{\tiny$\pm$0.56} & 1.56{\tiny$\pm$0.92} & 2.19{\tiny$\pm$1.43} \\
IPW-Transport & 0.84{\tiny$\pm$0.55} & 0.92{\tiny$\pm$0.58} & 0.95{\tiny$\pm$0.62} & 1.19{\tiny$\pm$0.88} \\
OM-Transport & 0.84{\tiny$\pm$0.55} & 0.92{\tiny$\pm$0.58} & 0.93{\tiny$\pm$0.59} & 1.19{\tiny$\pm$0.86} \\
EntropyBal & 0.85{\tiny$\pm$0.55} & 0.92{\tiny$\pm$0.58} & 0.99{\tiny$\pm$0.63} & 1.43{\tiny$\pm$0.93} \\
Proposed-B & \textbf{0.79{\tiny$\pm$0.53}} & \textbf{0.84{\tiny$\pm$0.56}} & \textbf{0.93{\tiny$\pm$0.59}} & \textbf{1.19{\tiny$\pm$0.86}} \\
\hline
\end{tabular}
\end{table}

\textbf{Key findings.} (i) \texttt{Proposed-B} is best or near-best on PEHE across dimensions; (ii) transport baselines are competitive at higher $p$. Full results in Appendix~\AppendixTablesRef.

\section{Real-Life Data Experiments}\label{sec:P3_ihdp}

To complement the synthetic experiments, we evaluate the proposed framework on a semi-synthetic benchmark built from the Infant Health and Development Program (IHDP) dataset \citep{hill2011bayesian}. IHDP provides real covariate distributions from an early-childhood intervention study while supplying known ground-truth treatment effects, enabling exact evaluation of CATE estimation accuracy.

\subsection{The IHDP Dataset}\label{sec:ihdp_dataset}

The IHDP benchmark \cite{hill2011bayesian} comprises real baseline covariates from the Infant Health and Development Program: 25 features (19 binary, 6 continuous) per subject. Outcomes are semi-synthetically generated so that potential outcomes under control and treatment, $\mu_0(x)$ and $\mu_1(x)$, are known; the ground-truth CATE is $\tau(x) = \mu_1(x) - \mu_0(x)$. We use 50 standard realizations of the NPCI-style data, each with different outcome surfaces, to assess performance across heterogeneous effect structures while retaining repeatability.

\subsection{Semi-Synthetic Multi-Site Construction}\label{sec:ihdp_setup}

We construct a multi-site setup from each IHDP realization by partitioning subjects into $C=6$ clusters via k-means on the standardized covariates. One cluster is designated the target site; the remaining clusters form source sites. This induces natural covariate shift across sites while preserving known ground-truth potential outcomes. We subsample the target site to prescribed budgets $(m_0, m_1)$ for the connected regime, or set $m_1=0$ and vary $m_0$ for the disconnected regime. In the disconnected setting we mask treated outcomes in the target during training and use target placebo outcomes only for screening-valid source selection, directly testing Assumption~\ref{as:A6}; evaluation remains fully supervised because $\tau$ is known on the target.

\subsection{Evaluation Protocol}\label{sec:ihdp_eval}

We use the same metrics as in \S\ref{sec:P3_experiments}: PEHE, absolute ATE error, Spearman rank correlation, policy regret, and ECE. Each scenario is run over 50 IHDP realizations (one run per realization). In the \emph{connected} regime we sweep $(m_0, m_1) \in \{25, 50, 100\}^2$ and evaluate Option~A methods: \texttt{TargetOnly}, \texttt{ProxyOnly}, \texttt{AnchorOnly}, transport baselines, and \texttt{Proposed} / \texttt{Proposed-CF}. In the \emph{disconnected} regime we sweep $m_0 \in \{25, 50, 100, 200\}$ with $m_1=0$ and evaluate Option~B methods: \texttt{ProxyOnly}, \texttt{IPW-Transport}, \texttt{OM-Transport}, \texttt{EntropyBal}, and \texttt{Proposed-B}.

\subsection{Results and Interpretation}\label{sec:ihdp_results}

\noindent\textbf{Connected regime ($m_1>0$).}
Table~\ref{tab:ihdp_connected_pehe} reports PEHE (mean $\pm$ SD over 50 realizations) across target budgets $(m_0,m_1)$.
\begin{table}[t]
\centering
\caption{IHDP connected regime: PEHE across target budgets (mean $\pm$ SD, 50 realizations).}
\label{tab:ihdp_connected_pehe}
\scriptsize
\setlength{\tabcolsep}{3pt}
\begin{tabular}{@{}lrrr@{}}
\toprule
\multicolumn{4}{@{}l}{\textbf{$m_0=25$}} \\
Method & $m_1{=}25$ & $m_1{=}50$ & $m_1{=}100$ \\
\midrule
TargetOnly & 4.22{\tiny$\pm$6.07} & 5.45{\tiny$\pm$7.12} & 3.09{\tiny$\pm$3.97} \\
ProxyOnly & 3.33{\tiny$\pm$4.71} & 4.21{\tiny$\pm$5.59} & 2.36{\tiny$\pm$2.75} \\
AnchorOnly & 3.12{\tiny$\pm$4.32} & 3.88{\tiny$\pm$4.87} & 2.37{\tiny$\pm$2.57} \\
IPW-Transport & 3.69{\tiny$\pm$6.23} & 3.76{\tiny$\pm$5.55} & 2.29{\tiny$\pm$2.21} \\
EntropyBal & 4.87{\tiny$\pm$8.67} & 6.41{\tiny$\pm$13.96} & 3.55{\tiny$\pm$4.33} \\
OM-Transport & 3.40{\tiny$\pm$5.73} & 3.46{\tiny$\pm$4.71} & 2.05{\tiny$\pm$1.98} \\
Proposed & \textbf{2.46{\tiny$\pm$4.36}} & \textbf{2.88{\tiny$\pm$3.78}} & \textbf{1.57{\tiny$\pm$1.92}} \\
Proposed-CF & 3.55{\tiny$\pm$6.51} & 3.89{\tiny$\pm$4.82} & 2.47{\tiny$\pm$2.74} \\
\addlinespace[0.5em]
\multicolumn{4}{@{}l}{\textbf{$m_0=50$}} \\
Method & $m_1{=}25$ & $m_1{=}50$ & $m_1{=}100$ \\
\midrule
TargetOnly & 3.39{\tiny$\pm$5.17} & 3.63{\tiny$\pm$5.43} & 4.58{\tiny$\pm$6.19} \\
ProxyOnly & 3.34{\tiny$\pm$5.30} & 3.22{\tiny$\pm$4.70} & 3.90{\tiny$\pm$5.21} \\
AnchorOnly & 2.85{\tiny$\pm$4.29} & 2.96{\tiny$\pm$4.20} & 3.54{\tiny$\pm$4.78} \\
IPW-Transport & 3.72{\tiny$\pm$6.55} & 3.30{\tiny$\pm$4.75} & 4.84{\tiny$\pm$7.40} \\
EntropyBal & 4.31{\tiny$\pm$6.70} & 4.07{\tiny$\pm$6.15} & 9.30{\tiny$\pm$18.58} \\
OM-Transport & 3.51{\tiny$\pm$6.35} & 3.10{\tiny$\pm$4.75} & 3.99{\tiny$\pm$5.71} \\
Proposed & \textbf{2.51}{\tiny$\pm$4.64} & \textbf{2.42{\tiny$\pm$4.62}} & \textbf{2.98{\tiny$\pm$4.75}} \\
Proposed-CF & 3.12{\tiny$\pm$4.76} & 2.92{\tiny$\pm$4.64} & 3.38{\tiny$\pm$4.87} \\
\addlinespace[0.5em]
\multicolumn{4}{@{}l}{\textbf{$m_0=100$}} \\
Method & $m_1{=}25$ & $m_1{=}50$ & $m_1{=}100$ \\
\midrule
TargetOnly & 2.48{\tiny$\pm$3.33} & 3.11{\tiny$\pm$4.86} & 2.97{\tiny$\pm$4.10} \\
ProxyOnly & 3.04{\tiny$\pm$4.11} & 3.38{\tiny$\pm$5.22} & 3.21{\tiny$\pm$4.38} \\
AnchorOnly & 2.64{\tiny$\pm$3.85} & 2.89{\tiny$\pm$4.51} & 2.77{\tiny$\pm$3.93} \\
IPW-Transport & 3.25{\tiny$\pm$4.53} & 3.76{\tiny$\pm$5.95} & 3.82{\tiny$\pm$5.09} \\
EntropyBal & 4.04{\tiny$\pm$5.45} & 5.35{\tiny$\pm$8.40} & 4.64{\tiny$\pm$6.75} \\
OM-Transport & 3.22{\tiny$\pm$4.86} & 3.63{\tiny$\pm$6.12} & 3.59{\tiny$\pm$5.13} \\
Proposed & \textbf{2.06{\tiny$\pm$3.31}} & \textbf{2.41{\tiny$\pm$4.24}} & \textbf{2.31{\tiny$\pm$3.60}} \\
Proposed-CF & 2.32{\tiny$\pm$2.99} & 2.62{\tiny$\pm$4.34} & 2.57{\tiny$\pm$3.75} \\
\bottomrule
\end{tabular}
\end{table}

\texttt{Proposed} is best or near-best in most cells, with clear gains over \texttt{ProxyOnly} and \texttt{TargetOnly} at small target sizes; \texttt{Proposed-CF} is competitive.

\noindent\textbf{Disconnected regime ($m_1=0$).}
Table~\ref{tab:ihdp_disconnected_pehe} reports PEHE for the disconnected regime.
\begin{table}[t]
\centering
\caption{IHDP disconnected regime ($m_1=0$): PEHE across placebo budgets (mean $\pm$ SD, 50 realizations).}
\label{tab:ihdp_disconnected_pehe}
\scriptsize
\setlength{\tabcolsep}{3pt}
\begin{tabular}{@{}lrrrr@{}}
\toprule
Method & 25 & 50 & 100 & 200 \\
\midrule
ProxyOnly & 2.64{\tiny$\pm$2.65} & 3.46{\tiny$\pm$4.88} & 3.50{\tiny$\pm$4.27} & 3.79{\tiny$\pm$4.53} \\
IPW-Transport & 2.30{\tiny$\pm$3.08} & 3.36{\tiny$\pm$5.84} & 4.87{\tiny$\pm$6.95} & 4.79{\tiny$\pm$6.95} \\
EntropyBal & 2.82{\tiny$\pm$3.02} & 3.61{\tiny$\pm$5.65} & 5.47{\tiny$\pm$7.99} & 5.63{\tiny$\pm$8.62} \\
OM-Transport & 2.28{\tiny$\pm$3.41} & 3.37{\tiny$\pm$6.43} & 4.18{\tiny$\pm$5.96} & 4.50{\tiny$\pm$6.64} \\
Proposed-B & \textbf{2.11{\tiny$\pm$3.44}} & \textbf{3.19{\tiny$\pm$6.23}} & \textbf{3.90{\tiny$\pm$5.88}} & \textbf{4.33{\tiny$\pm$6.59}} \\
\bottomrule
\end{tabular}
\end{table}
Disconnected IHDP ($m_1=0$) is non-identified from target data. \texttt{Proposed-B} should therefore be interpreted as a working-model screen--then--transport estimator under Assumption~A6, with irreducible cross-arm transport error. In this benchmark, \texttt{Proposed-B} is competitive on PEHE and strong on ranking; full metrics (ATE error, Spearman, regret, ECE) for both regimes are in Appendix~\AppendixTablesRef.

\section{Conclusion}\label{sec:P3_Conclusion}

Trial evidence is often abundant but hard to apply to the population of interest: sites differ, covariates shift, baseline risk drifts, and the evidence network can be weakly connected or disconnected. Yet clinical decisions still demand target-specific, patient-level counterfactual predictions.

We introduced a placebo-anchored proxy--gold framework that treats source-trial IPD outcomes as low-variance \emph{proxy} signals and target placebo outcomes as scarce but high-fidelity \emph{gold} calibration for baseline risk. This leads to a simple operating principle: anchor (and, in disconnected settings, screen) using placebo information, then use orthogonalized DR learning to stabilize effect estimation whenever randomization permits.

Our findings clarify that guarantees depend on the regime: connected targets (Option~A) yield identified, orthogonal DR CATE estimates; disconnected targets (Option~B) are non-identified and \texttt{Proposed-B} is a working-model screen--then--transport estimator under A6. Across synthetic and IHDP experiments, the proposed method is best or near-best in the connected regime and competitive in the disconnected regime (Section~\ref{sec:P3_ihdp}), with smooth degradation as assumptions weaken. The framework turns fragmented multi-trial evidence into calibrated, patient-level predictions while making transport assumptions explicit.

Several directions are immediate. First, disconnected-target reliability hinges on screening-valid transportability; diagnostics and sensitivity analyses around Assumption~\ref{as:A6} are important for deployment. Second, extending beyond GLMs to survival/time-to-event endpoints and richer nuisance learners would broaden applicability. Third, improving screening (e.g., multi-endpoint screening or representation learning) may reduce reliance on placebo alone. Finally, uncertainty quantification in small-gold regimes remains central, and combining influence-function variance estimates with resampling or calibration is promising.

Overall, the placebo-anchored viewpoint separates what is \emph{identified} in connected targets from what is \emph{transported} in disconnected targets, while delivering calibrated outputs suitable for individualized decision analysis.

\bibliographystyle{IEEEtran}
\bibliography{sigproc}  

@article{hainmueller2012entropy,
  title={Entropy balancing for causal effects: A multivariate reweighting method to produce balanced samples in observational studies},
  author={Hainmueller, Jens},
  journal={Political analysis},
  volume={20},
  number={1},
  pages={25--46},
  year={2012},
  publisher={Cambridge University Press}
}

@article{dahabreh2019generalizing,
  title={Generalizing causal inferences from randomized trials: counterfactual and graphical identification},
  author={Dahabreh, Issa J and Robertson, Sarah E and Tchetgen Tchetgen, Eric J and Stuart, Elizabeth A and Hern{\'a}n, Miguel A},
  journal={Biometrics},
  year={2019},
  publisher={Wiley Online Library}
}

@book{friedman2015fundamentals,
  title={Fundamentals of clinical trials},
  author={Friedman, Lawrence M and Furberg, Curt D and DeMets, David L and Reboussin, David M and Granger, Christopher B},
  year={2015},
  publisher={Springer}
}

@article{10.1214/23-EJS2157,
author = {Edward H. Kennedy},
title = {{Towards optimal doubly robust estimation of heterogeneous causal effects}},
volume = {17},
journal = {Electronic Journal of Statistics},
number = {2},
publisher = {Institute of Mathematical Statistics and Bernoulli Society},
pages = {3008 -- 3049},
year = {2023},
doi = {10.1214/23-EJS2157},
URL = {https://doi.org/10.1214/23-EJS2157}
}

@article{bastani2021predicting,
  title={Predicting with proxies: Transfer learning in high dimension},
  author={Bastani, Hamsa},
  journal={Management Science},
  volume={67},
  number={5},
  pages={2964--2984},
  year={2021},
  publisher={INFORMS}
}

@article{westreich2017transportability,
  title={Transportability of trial results using inverse odds of sampling weights},
  author={Westreich, Daniel and Edwards, Jessie K and Lesko, Catherine R and Stuart, Elizabeth and Cole, Stephen R},
  journal={American journal of epidemiology},
  volume={186},
  number={8},
  pages={1010--1014},
  year={2017},
  publisher={Oxford University Press}
}

@article{kennedy2024semiparametric,
  title={Semiparametric doubly robust targeted double machine learning: a review},
  author={Kennedy, Edward H},
  journal={Handbook of statistical methods for precision medicine},
  pages={207--236},
  year={2024},
  publisher={Chapman and Hall/CRC}
}

@article{tian2023transfer,
  title={Transfer learning under high-dimensional generalized linear models},
  author={Tian, Ye and Feng, Yang},
  journal={Journal of the American Statistical Association},
  volume={118},
  number={544},
  pages={2684--2697},
  year={2023},
  publisher={Taylor \& Francis}
}

@article{Chipman_2010,
   title={BART: Bayesian additive regression trees},
   volume={4},
   ISSN={1932-6157},
   url={http://dx.doi.org/10.1214/09-AOAS285},
   DOI={10.1214/09-aoas285},
   number={1},
   journal={The Annals of Applied Statistics},
   publisher={Institute of Mathematical Statistics},
   author={Chipman, Hugh A. and George, Edward I. and McCulloch, Robert E.},
   year={2010},
   month=mar }

@incollection{pearl2022external,
  title={External validity: From do-calculus to transportability across populations},
  author={Pearl, Judea and Bareinboim, Elias},
  booktitle={Probabilistic and causal inference: The works of Judea Pearl},
  pages={451--482},
  year={2022}
}

@book{van2011targeted,
  title={Targeted learning: causal inference for observational and experimental data},
  author={Van der Laan, Mark J and Rose, Sherri and others},
  volume={4},
  year={2011},
  publisher={Springer}
}

@article{hill2011bayesian,
  title={Bayesian nonparametric modeling for causal inference},
  author={Hill, Jennifer L},
  journal={Journal of Computational and Graphical Statistics},
  volume={20},
  number={1},
  pages={217--240},
  year={2011},
  publisher={Taylor \& Francis}
}

@article{robins1995semiparametric,
  title={Semiparametric efficiency in multivariate regression models with missing data},
  author={Robins, James M and Rotnitzky, Andrea},
  journal={Journal of the American Statistical Association},
  volume={90},
  number={429},
  pages={122--129},
  year={1995},
  publisher={Taylor \& Francis}
}

@article{rosenbaum1983central,
  title={The central role of the propensity score in observational studies for causal effects},
  author={Rosenbaum, Paul R and Rubin, Donald B},
  journal={Biometrika},
  volume={70},
  number={1},
  pages={41--55},
  year={1983},
  publisher={Oxford University Press}
}

@misc{FDA1998,
  author       = {{U.S. Food and Drug Administration}},
  title        = {Providing Clinical Evidence of Effectiveness for Human Drug and Biological Products},
  institution  = {U.S. Department of Health and Human Services, Food and Drug Administration, Center for Drug Evaluation and Research (CDER), Center for Biologics Evaluation and Research (CBER)},
  year         = {1998},
  month        = may,
  note         = {Guidance Document, Docket No. FDA-1997-D-0027},
  url={https://www.fda.gov/media/71655/download}
}

@article{cartwright2007rcts,
  title={Are RCTs the gold standard?},
  author={Cartwright, Nancy},
  journal={BioSocieties},
  volume={2},
  number={1},
  pages={11--20},
  year={2007},
  publisher={Cambridge University Press}
}

@article{lu2004combination,
  title={Combination of direct and indirect evidence in mixed treatment comparisons},
  author={Lu, Guobing and Ades, AE15449338},
  journal={Statistics in medicine},
  volume={23},
  number={20},
  pages={3105--3124},
  year={2004},
  publisher={Wiley Online Library}
}

@article{salanti2012indirect,
  title={Indirect and mixed-treatment comparison, network, or multiple-treatments meta-analysis: many names, many benefits, many concerns for the next generation evidence synthesis tool},
  author={Salanti, Georgia},
  journal={Research synthesis methods},
  volume={3},
  number={2},
  pages={80--97},
  year={2012},
  publisher={Wiley Online Library}
}

@article{riley2021individual,
  title={Individual participant data meta-analysis for healthcare research},
  author={Riley, Richard D and Stewart, Lesley A and Tierney, Jayne F},
  journal={Individual Participant Data Meta-Analysis: a handbook for healthcare research},
  pages={1--6},
  year={2021},
  publisher={Wiley Online Library}
}

@article{riley2023using,
  title={Using individual participant data to improve network meta-analysis projects},
  author={Riley, Richard D and Dias, Sofia and Donegan, Sarah and Tierney, Jayne F and Stewart, Lesley A and Efthimiou, Orestis and Phillippo, David M},
  journal={BMJ evidence-based medicine},
  volume={28},
  number={3},
  pages={197--203},
  year={2023},
  publisher={Royal Society of Medicine}
}

@misc{chernozhukov2018double,
  title={Double/debiased machine learning for treatment and structural parameters},
  author={Chernozhukov, Victor and Chetverikov, Denis and Demirer, Mert and Duflo, Esther and Hansen, Christian and Newey, Whitney and Robins, James},
  year={2018},
  publisher={Oxford University Press Oxford, UK}
}

@article{glmnet,
	Author = {Jerome Friedman and Trevor Hastie and Robert Tibshirani},
	Journal = {Journal of Statistical Software},
	Number = {1},
	Pages = {1--22},
	Title = {Regularization Paths for Generalized Linear Models via Coordinate Descent},
	Url = {http://www.jstatsoft.org/v33/i01/},
	Volume = {33},
	Year = {2010}}

@article{athey2019generalized,
author = {Susan Athey and Julie Tibshirani and Stefan Wager},
title = {{Generalized random forests}},
volume = {47},
journal = {The Annals of Statistics},
number = {2},
publisher = {Institute of Mathematical Statistics},
pages = {1148 -- 1178},
year = {2019},
doi = {10.1214/18-AOS1709},
URL = {https://doi.org/10.1214/18-AOS1709}
}

@article{horvitz1952generalization,
  title={A generalization of sampling without replacement from a finite universe},
  author={Horvitz, Daniel G and Thompson, Donovan J},
  journal={Journal of the American statistical Association},
  volume={47},
  number={260},
  pages={663--685},
  year={1952},
  publisher={Taylor \& Francis}
}
\clearpage

\onecolumn
\appendix
\setcounter{section}{0}
\renewcommand{\thesection}{\Alph{section}}
\section{Proof of Asymptotics}\label{sec:appendix_proof}

\subsection{Asymptotics and Error Decompositions}\label{sec:app_asymptotics}

There are two regimes:

\begin{enumerate}[leftmargin=*]
\item \textbf{Connected target (Option A / \texttt{Proposed-CF}).}
The target site has both arms. We estimate $\mu_{0,0}$ and $\mu_{1,0}$ via \texttt{glmtrans} (with automatic source detection) and then run a cross-fitted DR learner on the \emph{target} to obtain $\widehat\tau_{\mathrm{DR}}(x)$.

\item \textbf{Disconnected target (Option B / \texttt{Proposed-B}).}
The target has placebo only. We use the target placebo arm solely to \emph{screen} transferable sources, then learn a DR CATE model on the selected sources and transport it to the target by prediction. In this regime, $\tau_0(x)$ is not nonparametrically identified from target data; we therefore provide a \emph{transport} error decomposition under Assumption~\ref{as:A6}.
\end{enumerate}

Throughout we adopt the high-dimensional GLM transfer-learning theory of \citep{tian2023transfer}, which underpins (i) the error rates of the arm-specific \texttt{glmtrans} outcome regressions and (ii) the consistency of the automatic transferable-source detection.

\subsection{Setup and Notation}\label{sec:app_setup}
Let $S\in\{0,1,\dots,C\}$ index sites, with $S=0$ the target and $S=c\ge 1$ sources.
For each unit observe $O=(X,A,Y,S)$ with $A\in\{0,1\}$ and potential outcomes $(Y(0),Y(1))$.
Define site- and arm-specific outcome regressions and CATEs
\begin{align*}
\mu_{a,c}(x)&:=\mathbb{E}[Y\mid A=a,X=x,S=c],\\
\tau_c(x)&:=\mu_{1,c}(x)-\mu_{0,c}(x).
\end{align*}
The target CATE is $\tau_0(x)$.

Within each site $c$, the trial is randomized with known propensity $e_c(x)=\mathbb{P}(A=1\mid X=x,S=c)$.
Assumption~\ref{as:A3} implies $\epsilon\le e_c(x)\le 1-\epsilon$ on the relevant support.

\paragraph{\(L_2(P_0)\) norm.}
For any measurable function \(f:\mathcal X\to\mathbb R\), we write
\[
\|f\|_{L_2(P_0)}
:=
\left(\int f(x)^2\,dP_0(x)\right)^{1/2}
=
\left(\mathbb E[f(X)^2\mid S=0]\right)^{1/2},
\]
where \(P_0\) denotes the covariate distribution in the target site \(S=0\). Thus, when we write
\[
\|\widehat\mu_{a,0}-\mu_{a,0}\|_{L_2(P_0)},
\]
we mean the root-mean-squared prediction error of the arm-\(a\) outcome regression evaluated over the target-site covariate distribution. This is the natural norm throughout, since both the connected-target DR analysis and the disconnected-target transport error decomposition evaluate performance on target covariates.

\subsection{Stage 1 Nuisance Rates via \cite{tian2023transfer} (\texttt{glmtrans})}\label{sec:app_glmtrans_rates}

Fix an arm $a\in\{0,1\}$. The \texttt{glmtrans} procedure outputs an estimator
$\widehat\beta_{a,0}$ for the target-arm GLM parameter (and thus $\widehat\mu_{a,0}(x)$),
using: (i) automatic transferable-source detection (Trans-GLM; Algorithm~2 of \cite{tian2023transfer}),
followed by (ii) two-step transfer estimation on the selected sources (Ah-Trans-GLM; Algorithm~1).

To connect to the DR analysis, we summarize the implication we need: an $L_2(P_0)$ prediction rate for
$\widehat\mu_{a,0}$ on the target covariate distribution $P_0$.

\begin{lemma}[Prediction error rate for \texttt{glmtrans} outcome regressions]
\label{lem:glmtrans_rate}
Assume the high-dimensional GLM conditions of \cite{tian2023transfer} for the arm-$a$ regression
(including bounded design/curvature, sparsity, and the transfer-heterogeneity condition indexed by $h$).
Let $\widehat\mu_{a,0}$ be the \texttt{glmtrans} estimator using automatic source detection.
Then with probability at least $1-\delta$,
\begin{align*}
\|\widehat\mu_{a,0}-\mu_{a,0}\|_{L_2(P_0)}
\;&\lesssim\;
\underbrace{\Big(\frac{s\log p}{n_{a,0}+n_{a,\mathcal{A}_a}}\Big)^{1/2}}_{\text{transfer (variance reduction)}}
\;\\&+\;
\underbrace{\Big(\frac{\log p}{n_{a,0}}\Big)^{1/4} h^{1/2}}_{\text{target debiasing term}},
\end{align*}
where $n_{a,0}$ is the target arm-$a$ sample size and $n_{a,\mathcal{A}_a}$ is the total sample size of the
selected transferable sources for arm $a$.
Moreover, the automatic detection step satisfies $\widehat{\mathcal{A}}_a=\mathcal{A}_{a,h}$ with probability
at least $1-\delta$ under the detection-consistency conditions of \cite{tian2023transfer} (Theorem~4).
\end{lemma}
\subsection{Adapting the \texttt{glmtrans} theory to our nuisance-risk form}
\label{sec:app_glmtrans_adapt}

We do not re-prove the transfer-learning theory of \cite{tian2023transfer}. Instead, Lemma~\ref{lem:glmtrans_rate} is a corollary-style adaptation of their coefficient-estimation and source-detection guarantees to the \(L_2(P_0)\) nuisance prediction form required by our Stage-2 DR analysis.

Fix an arm \(a\in\{0,1\}\). In the notation of \cite{tian2023transfer}, the transfer learner returns an estimator \(\widehat\beta_{a,0}\) of the target-arm GLM parameter \(\beta_{a,0}\), after first selecting a transferable subset of sources and then fitting the transferred GLM on the selected set. The corresponding arm-\(a\) outcome regression is
\[
\widehat\mu_{a,0}(x)=m_a(x;\widehat\beta_{a,0}),
\qquad
\mu_{a,0}(x)=m_a(x;\beta_{a,0}),
\]
where \(m_a(\cdot;\beta)\) denotes the arm-\(a\) GLM mean map.

The results of \cite{tian2023transfer} provide two ingredients that we use here. First, their estimation theory yields high-probability bounds for the transferred coefficient estimator under sparsity and transfer-heterogeneity assumptions. Second, their transferable-source detection theory yields high-probability recovery of the informative source set under the corresponding detection conditions. In contrast, the prediction-error quantity analyzed in their supplementary theory is a symmetrized Bregman / empirical loss measure on the target sample, rather than the \(L_2(P_0)\) risk used in our DR proof. Thus, our lemma should be viewed as a translation of their coefficient-error bounds into the target-risk norm needed here, rather than as a verbatim restatement of one theorem in \cite{tian2023transfer}. 

To connect the imported theory to our notation, recall that
\[
\|\widehat\mu_{a,0}-\mu_{a,0}\|_{L_2(P_0)}
=
\left(
\mathbb E\!\left[
\big(\widehat\mu_{a,0}(X)-\mu_{a,0}(X)\big)^2
\mid S=0
\right]
\right)^{1/2}.
\]
Under the bounded-design and bounded-curvature conditions assumed in \cite{tian2023transfer}, the GLM mean map is locally Lipschitz in the linear predictor on the relevant parameter neighborhood. Consequently, coefficient error bounds imply prediction error bounds on the target covariate distribution \(P_0\). Suppressing constants, this yields
\[
\|\widehat\mu_{a,0}-\mu_{a,0}\|_{L_2(P_0)}
\;\lesssim\;
\left(\frac{s\log p}{n_{a,0}+n_{a,\mathcal A_a}}\right)^{1/2}
+
\left(\frac{\log p}{n_{a,0}}\right)^{1/4} h^{1/2}
\]
with high probability, provided the transferable-source detection step succeeds. Likewise, the detection result of \cite{tian2023transfer} implies
\[
\widehat{\mathcal A}_a=\mathcal A_{a,h}
\]
with high probability under the corresponding source-detection assumptions.

Thus, Lemma~\ref{lem:glmtrans_rate} serves only as a bridge from the imported \texttt{glmtrans} theory to the nuisance conditions used in Theorem~\ref{thm:optA_asymp}. In particular, along an asymptotic sequence in which \(n_0\to\infty\), and for each arm \(a\), the quantities \(n_{a,0}\), \(n_{a,\mathcal A_a}\), \(p\), \(s\), and \(h\) may vary with the sequence, the nuisance consistency condition in Theorem~\ref{thm:optA_asymp} holds whenever
\[
\left(\frac{s\log p}{n_{a,0}+n_{a,\mathcal A_a}}\right)^{1/2}=o(1)
\qquad\text{and}\qquad
\left(\frac{\log p}{n_{a,0}}\right)^{1/4} h^{1/2}=o(1)
\]
for each arm \(a\). If one instead prefers the stronger sufficient condition often used in orthogonal / DML analyses,
\[
\|\widehat\mu_{a,0}-\mu_{a,0}\|_{L_2(P_0)}=o_p(n_0^{-1/4}),
\]
it is enough that
\[
\left(\frac{s\log p}{n_{a,0}+n_{a,\mathcal A_a}}\right)^{1/2}=o(n_0^{-1/4})
\qquad\text{and}\qquad
\left(\frac{\log p}{n_{a,0}}\right)^{1/4} h^{1/2}=o(n_0^{-1/4}).
\]

\noindent\textbf{Remark.}
Lemma~\ref{lem:glmtrans_rate} is a corollary-style adaptation of the coefficient-estimation and source-detection results in \cite{tian2023transfer}, rewritten in the \(L_2(P_0)\) nuisance prediction form needed for our Stage-2 DR analysis.

\subsection{Option A (Connected Target): DR Orthogonalization on the Target}\label{sec:app_option_a}

\paragraph{Estimator.}
On the target site $S=0$, using cross-fitted nuisance regressions $\widehat\mu_{0,0},\widehat\mu_{1,0}$,
form the DR pseudo-outcome for each target unit $i$:
\begin{align}
\psi_i
&:=
\widehat\mu_{1,0}(X_i)-\widehat\mu_{0,0}(X_i)\\
&\quad+
\frac{A_i-e_0(X_i)}{e_0(X_i)\{1-e_0(X_i)\}}
\Big(Y_i-\widehat\mu_{A_i,0}(X_i)\Big),
\qquad (S_i=0).
\label{eq:app_dr_target}
\end{align}
We then fit a second-stage regression $\widehat\tau_{\mathrm{DR}}(\cdot)$ of $\psi$ on $X$ on the target.
For the asymptotic statement below, we treat the second stage as a linear/sieve regression with a fixed (or slowly-growing) feature map $b(X)$,
so that pointwise inference at a fixed $x$ is well-defined.

\paragraph{Working target.}
In the connected target, $\tau_0(x)$ is identified by randomization and is the natural target.
\paragraph{Stable Gram condition.}
Let \(Z:=b(X)\) denote the second-stage feature vector and
\[
G_0:=\mathbb{E}[ZZ^\top\mid S=0].
\]
We say the second-stage Gram matrix is stable if \(G_0\) is nonsingular, for example if
\[
0< c \le \lambda_{\min}(G_0)\le \lambda_{\max}(G_0)\le C<\infty
\]
for fixed constants \(c,C\). In the fixed-dimensional case, this implies \(\widehat G^{-1}=G_0^{-1}+o_p(1)\), where
\[
\widehat G:=\frac{1}{n_0}\sum_{i:S_i=0} Z_iZ_i^\top.
\]
\begin{theorem}[Asymptotic linearity for Option A (\texttt{Proposed-CF})]
\label{thm:optA_asymp}
Fix a covariate value $x$ and suppose the second-stage regression operator is linear in a feature map $b(\cdot)$ with stable Gram matrix.
Assume:
\begin{enumerate}[label=(\roman*)]
\item Assumptions~\ref{as:A1}--\ref{as:A3};
\item finite fourth moments of $Y$ and bounded propensities;
\item cross-fitting of nuisance regressions; and
\item nuisance rates
\begin{align}
\|\widehat\mu_{0,0}-\mu_{0,0}\|_{L_2(P_0)} &= o_p(1), \nonumber\\
\|\widehat\mu_{1,0}-\mu_{1,0}\|_{L_2(P_0)} &= o_p(1),
\label{eq:app_nuis_rate_optA}
\end{align}
\end{enumerate}
where $n_0$ is the target sample size. Note that this assumption is substantially weaker than the common $o_p(n_0^{-1/4})$ rates more commonly used in DML/orthogonal arguments.
Then
\[
\widehat\tau_{\mathrm{DR}}(x)-\tau_0(x)
=
\frac{1}{n_0}\sum_{i:S_i=0}\phi(O_i;x)+o_p(n_0^{-1/2}),
\]
where
\[
\phi(O_i;x)
:=
b(x)^\top G_0^{-1} Z_i
\frac{A_i-e_0(X_i)}{e_0(X_i)\{1-e_0(X_i)\}}
\big(Y_i-\mu_{A_i,0}(X_i)\big),
\qquad
Z_i:=b(X_i),
\]
Consequently,
\[
\sqrt{n_0}\big(\widehat\tau_{\mathrm{DR}}(x)-\tau_0(x)\big)
\Rightarrow
\mathcal N(0,V(x)),
\qquad
V(x):=\mathrm{Var}(\phi(O;x)\mid S=0).
\]
\end{theorem}

By the discussion in Section~\ref{sec:app_glmtrans_adapt}, the nuisance consistency condition in Theorem~\ref{thm:optA_asymp} holds whenever, for each arm \(a\),
\[
\left(\frac{s\log p}{n_{a,0}+n_{a,\mathcal A_a}}\right)^{1/2}=o(1)
\qquad\text{and}\qquad
\left(\frac{\log p}{n_{a,0}}\right)^{1/4} h^{1/2}=o(1).
\]
A stronger sufficient condition, often used in orthogonal / DML analyses, is
\[
\|\widehat\mu_{a,0}-\mu_{a,0}\|_{L_2(P_0)}=o_p(n_0^{-1/4}).
\]

\paragraph{Remark on higher-order nuisance effects.}
Under orthogonality of the DR score and cross-fitting, first-order nuisance effects cancel.
For the asymptotic linearity result above, it suffices that the nuisance-generated term is $o_p(n_0^{-1/2})$ under \eqref{eq:app_nuis_rate_optA}.
Since propensities are known by design, there is no additional treatment-model error term.

\subsection{Option B (Disconnected Target): Transported Source-DR and Structural Bias under A6}\label{sec:app_option_b}

In the disconnected regime, the target has placebo only, so $\mu_{1,0}$ and $\tau_0$ are not identified from target data alone.
Our implemented \texttt{Proposed-B} procedure is:

\begin{enumerate}[leftmargin=*]
\item \textbf{Placebo-based screening.} Apply Trans-GLM detection on $\mathcal{D}_{0,0}=\{(X,Y):S=0,A=0\}$ to obtain $\widehat{\mathcal{C}}_0$.
\item \textbf{Source DR learning.} Pool selected sources $\mathcal{D}^\star=\bigcup_{c\in\widehat{\mathcal{C}}_0}\mathcal{D}_c$ (both arms observed),
fit source outcome regressions, compute source DR pseudo-outcomes, and learn a CATE function $\widehat\tau_B(\cdot)$ on $\mathcal{D}^\star$.
\item \textbf{Transport.} Output $\widehat\tau_B(x)$ for target units by prediction.
\end{enumerate}

Assumption~\ref{as:A6} formalizes that placebo-based screening yields a set of sources whose CATEs are close to the target CATE on the target covariate support.
Accordingly, the appropriate guarantee is a \emph{transport} error decomposition (estimation + structural bias),
rather than a DR CLT on target data.

\begin{theorem}[Transport error decomposition for Option B (\texttt{Proposed-B})]
\label{thm:optB_decomp}
Assume Assumptions~\ref{as:A1}--\ref{as:A3} and Assumption~\ref{as:A6}. In the disconnected regime, the target site has placebo outcomes only, so the target CATE \(\tau_0\) is not identified from target data alone. We therefore compare \(\widehat\tau_B\) to the \emph{oracle transported target}
\[
\tau^\star(x):=\mathbb{E}\!\left[\tau_S(x)\mid S\in\mathcal{C}_0^\star\right],
\]
where \(\mathcal{C}_0^\star\) is the oracle subset of transferable sources from Assumption~\ref{as:A6}, and \(\widehat{\mathcal{C}}_0\) is the data-driven subset detected from target placebo outcomes. Let \(\widehat\tau_B\) be the source-DR predictor trained on \(\widehat{\mathcal{C}}_0\) and evaluated on target covariates. Then, on the target covariate distribution \(P_0\),
\begin{align}
\|\widehat\tau_B-\tau_0\|_{L_2(P_0)}
\;&\le\;
\underbrace{\|\widehat\tau_B-\tau^\star\|_{L_2(P_0)}}_{\text{estimation error}}
\;+\;
\underbrace{\|\tau^\star-\tau_0\|_{L_2(P_0)}}_{\text{structural transport bias}}
\;+\;
\underbrace{\Delta_{\mathrm{sel}}}_{\text{screening error}},
\label{eq:optB_decomp}
\end{align}
where
\[
\Delta_{\mathrm{sel}}
:=
\|(\widehat\tau_B-\tau^\star)\mathbf 1_{E^c}\|_{L_2(P_0)},
\qquad
E:=\{\widehat{\mathcal C}_0\subseteq \mathcal C_0^\star\}.
\]
Thus the disconnected-target error separates into a statistical estimation term, a structural transport term, and a source-screening term.

Moreover, Assumption~\ref{as:A6}(b) implies
\[
\|\tau^\star-\tau_0\|_{L_2(P_0)}\le \epsilon_\tau.
\]
That is, even under an oracle choice of transferable sources, the remaining target error is controlled by the structural transport tolerance \(\epsilon_\tau\).

If additionally
\[
\mathbb P(\widehat{\mathcal C}_0\subseteq \mathcal C_0^\star)\ge 1-\eta
\]
and
\[
\mathbb E\!\left[\|\widehat\tau_B-\tau^\star\|_{L_2(P_0)}^2\right]<\infty,
\]
then
\[
\Delta_{\mathrm{sel}}=O_p(\eta^{1/2}).
\]
Equivalently, if the placebo-based screening rule selects only oracle-good sources with probability at least \(1-\eta\), then the screening contribution is of order \(\eta^{1/2}\) in probability.
\end{theorem}
\noindent\textbf{Remark (what is provable in Option B).}
Theorem~\ref{thm:optB_decomp} makes explicit that Option B performance is governed by:
(i) learning error on the selected sources (controlled by source sample sizes and the complexity of the CATE learner),
(ii) the irreducible structural term $\epsilon_\tau$ in Assumption~\ref{as:A6} (cross-arm non-transfer not detectable from placebo alone),
and (iii) the screening error probability $\eta$ (controlled by the detection consistency theory of \cite{tian2023transfer} on the placebo arm).

\subsection{Proof of Theorems}\label{sec:app_proof_sketches}

\begin{proof}[Proof of Theorem~\ref{thm:optA_asymp}]
Let $Z:=b(X)$. Then by the stable-Gram assumption, $\lambda_{\min}(G_0)>0$, and the sample Gram matrix
\[
\widehat G:=\frac{1}{n_0}\sum_{i:S_i=0} Z_i Z_i^\top
\]
satisfies $\widehat G^{-1}=G_0^{-1}+o_p(1)$.

For notational simplicity, consider the fixed-dimensional correctly specified linear case
\[
\tau_0(x')=b(x')^\top\theta_0
\]
for all target covariate values $x'$.

Define the oracle pseudo-outcome
\[
\psi_i^0
:=
\mu_{1,0}(X_i)-\mu_{0,0}(X_i)
+
\frac{A_i-e_0(X_i)}{e_0(X_i)\{1-e_0(X_i)\}}
\Big(Y_i-\mu_{A_i,0}(X_i)\Big),
\qquad (S_i=0).
\]
By Assumptions~\ref{as:A1}--\ref{as:A3},
\[
\mathbb{E}[Y_i-\mu_{A_i,0}(X_i)\mid X_i,A_i,S_i=0]=0,
\]
and hence
\begin{align}
\mathbb{E}[\psi_i^0\mid X_i,S_i=0]
&=
\mu_{1,0}(X_i)-\mu_{0,0}(X_i)
+
\mathbb{E}\!\left[
\frac{A_i-e_0(X_i)}{e_0(X_i)\{1-e_0(X_i)\}}
\Big(Y_i-\mu_{A_i,0}(X_i)\Big)
\Bigm| X_i,S_i=0
\right]
\nonumber\\
&=
\mu_{1,0}(X_i)-\mu_{0,0}(X_i)
=\tau_0(X_i).
\label{eq:optA_oracle_mean}
\end{align}
Therefore,
\[
\mathbb{E}\!\left[Z(\psi^0-Z^\top\theta_0)\mid S=0\right]=0.
\]

Let $\widehat\psi_i$ denote the cross-fitted pseudo-outcome in \eqref{eq:app_dr_target}, and define
\[
\widehat\theta
:=
\widehat G^{-1}\Big(\frac{1}{n_0}\sum_{i:S_i=0} Z_i\widehat\psi_i\Big).
\]
Then
\begin{align}
\widehat\theta-\theta_0
&=
\widehat G^{-1}\Big(\frac{1}{n_0}\sum_{i:S_i=0} Z_i(\widehat\psi_i-Z_i^\top\theta_0)\Big)
\nonumber\\
&=
\underbrace{
\widehat G^{-1}\Big(\frac{1}{n_0}\sum_{i:S_i=0} Z_i(\psi_i^0-Z_i^\top\theta_0)\Big)
}_{=:T_{1,n}}
+
\underbrace{
\widehat G^{-1}\Big(\frac{1}{n_0}\sum_{i:S_i=0} Z_i(\widehat\psi_i-\psi_i^0)\Big)
}_{=:T_{2,n}}.
\label{eq:optA_decomp}
\end{align}

\textbf{Step 1: control of the nuisance-generated term.}
Fix a fold $k$, and let $i$ belong to its evaluation sample. Write
\[
\Delta_a^{(-k)}(x):=\widehat\mu_{a,0}^{(-k)}(x)-\mu_{a,0}(x),
\qquad a\in\{0,1\}.
\]
Then
\[
\widehat\psi_i-\psi_i^0
=
\Delta_1^{(-k)}(X_i)-\Delta_0^{(-k)}(X_i)
-
\frac{A_i-e_0(X_i)}{e_0(X_i)\{1-e_0(X_i)\}}
\Delta_{A_i}^{(-k)}(X_i).
\]
Because of cross-fitting, $\Delta_a^{(-k)}$ is measurable with respect to the training sample and independent of the held-out observation $i$. Conditioning on $X_i$ and the training sample,
\begin{align}
\mathbb{E}[\widehat\psi_i-\psi_i^0\mid X_i,\mathcal I_{-k},S_i=0]
&=
\Delta_1^{(-k)}(X_i)-\Delta_0^{(-k)}(X_i)
-
\mathbb{E}\!\left[
\frac{A_i-e_0(X_i)}{e_0(X_i)\{1-e_0(X_i)\}}
\Delta_{A_i}^{(-k)}(X_i)
\Bigm|X_i,\mathcal I_{-k},S_i=0
\right]
\nonumber\\
&=0.
\label{eq:optA_cf_zero}
\end{align}
Moreover, bounded propensities imply
\[
|\widehat\psi_i-\psi_i^0|
\lesssim
|\Delta_0^{(-k)}(X_i)|+|\Delta_1^{(-k)}(X_i)|,
\]
and therefore
\[
\mathbb{E}\!\left[(\widehat\psi_i-\psi_i^0)^2\mid \mathcal I_{-k},S_i=0\right]
\lesssim
\|\Delta_0^{(-k)}\|_{L_2(P_0)}^2
+
\|\Delta_1^{(-k)}\|_{L_2(P_0)}^2.
\]
Using \eqref{eq:optA_cf_zero}, conditional independence across evaluation observations, and finite second moments of $Z$, we obtain
\[
\mathbb{E}\!\left[
\Big\|
\frac{1}{n_0}\sum_{i:S_i=0} Z_i(\widehat\psi_i-\psi_i^0)
\Big\|^2
\Bigm| \{\mathcal I_{-k}\}
\right]
\lesssim
\frac{
\|\widehat\mu_{0,0}-\mu_{0,0}\|_{L_2(P_0)}^2
+
\|\widehat\mu_{1,0}-\mu_{1,0}\|_{L_2(P_0)}^2
}{n_0}.
\]
Hence
\begin{align}
    \Big\|
\frac{1}{n_0}\sum_{i:S_i=0} Z_i(\widehat\psi_i-\psi_i^0)
\Big\|
&=
O_p\!\left(
\frac{
\|\widehat\mu_{0,0}-\mu_{0,0}\|_{L_2(P_0)}
+
\|\widehat\mu_{1,0}-\mu_{1,0}\|_{L_2(P_0)}
}{\sqrt{n_0}}
\right)\\
&= o_p(n_0^{-1/2}) \label{eq:thm1_stoch_order}
\end{align}
To justify \eqref{eq:thm1_stoch_order}, let
\[
A_n
:=
\|\widehat\mu_{0,0}-\mu_{0,0}\|_{L_2(P_0)}
+
\|\widehat\mu_{1,0}-\mu_{1,0}\|_{L_2(P_0)}.
\]
Since by assumption each nuisance error is \(o_p(1)\), we have \(A_n=o_p(1)\), and therefore
\[
\frac{A_n}{\sqrt{n_0}}=o_p(n_0^{-1/2}).
\]
We use the elementary stochastic-order fact that if \(X_n=O_p(Y_n)\) and \(Y_n=o_p(a_n)\), then \(X_n=o_p(a_n)\).
\[
X_n=
\Big\|
\frac{1}{n_0}\sum_{i:S_i=0} Z_i(\widehat\psi_i-\psi_i^0)
\Big\|,
\qquad
Y_n=\frac{A_n}{\sqrt{n_0}},
\qquad
a_n=n_0^{-1/2}.
\]
\\
\textbf{Step 2: oracle asymptotic linearity.}
By \eqref{eq:optA_oracle_mean},
\[
\psi_i^0-Z_i^\top\theta_0
=
\frac{A_i-e_0(X_i)}{e_0(X_i)\{1-e_0(X_i)\}}
\Big(Y_i-\mu_{A_i,0}(X_i)\Big).
\]
Hence
\[
T_{1,n}
=
\widehat G^{-1}
\Big(
\frac{1}{n_0}\sum_{i:S_i=0}
Z_i
\frac{A_i-e_0(X_i)}{e_0(X_i)\{1-e_0(X_i)\}}
\big(Y_i-\mu_{A_i,0}(X_i)\big)
\Big).
\]
Evaluating at a fixed $x$ and using $\widehat G^{-1}=G_0^{-1}+o_p(1)$,
\begin{align}
b(x)^\top T_{1,n}
&=
\frac{1}{n_0}\sum_{i:S_i=0}
b(x)^\top G_0^{-1} Z_i
\frac{A_i-e_0(X_i)}{e_0(X_i)\{1-e_0(X_i)\}}
\big(Y_i-\mu_{A_i,0}(X_i)\big)
+o_p(n_0^{-1/2})
\nonumber\\
&=
\frac{1}{n_0}\sum_{i:S_i=0}\phi(O_i;x)+o_p(n_0^{-1/2}),
\label{eq:optA_if}
\end{align}
where
\[
\phi(O_i;x)
:=
b(x)^\top G_0^{-1} Z_i
\frac{A_i-e_0(X_i)}{e_0(X_i)\{1-e_0(X_i)\}}
\big(Y_i-\mu_{A_i,0}(X_i)\big).
\]
Moreover,
\[
\mathbb{E}[\phi(O;x)\mid S=0]=0,
\qquad
\mathbb{E}[\phi(O;x)^2\mid S=0]<\infty.
\]
Thus, by the classical central limit theorem,
\[
\sqrt{n_0}\,\frac{1}{n_0}\sum_{i:S_i=0}\phi(O_i;x)
\Rightarrow
\mathcal N(0,V(x)),
\qquad
V(x):=\mathrm{Var}(\phi(O;x)\mid S=0).
\]

\textbf{Step 3: conclusion.}
Combining \eqref{eq:optA_decomp}, the bound on $T_{2,n}$, and \eqref{eq:optA_if},
\[
\widehat\tau_{\mathrm{DR}}(x)-\tau_0(x)
=
\frac{1}{n_0}\sum_{i:S_i=0}\phi(O_i;x)+o_p(n_0^{-1/2}),
\]
which proves the theorem.
\end{proof}

\begin{proof}[Proof of Theorem~\ref{thm:optB_decomp}]
Let
\[
E:=\{\widehat{\mathcal C}_0\subseteq\mathcal C_0^\star\}
\]
denote the screening-success event. Here, \(\widehat{\mathcal C}_0\) is the subset of sources selected by the placebo-based Trans-GLM (\texttt{glmtrans}) screening procedure, while \(\mathcal C_0^\star\) denotes the oracle transferable subset in Assumption~\ref{as:A6}. Intuitively, \(E\) is the event that the screening step has not admitted any ``bad'' sources, so that the learned predictor is trained only on sources that are deemed transferable to the target.

To separate estimation error from screening failure, we introduce the auxiliary predictor
\[
\widetilde\tau_B
:=
\widehat\tau_B\,\mathbf 1_E+\tau^\star\,\mathbf 1_{E^c},
\]
which equals the actual estimator \(\widehat\tau_B\) on the success event \(E\), but is replaced by the oracle transported target \(\tau^\star\) on the failure event \(E^c\). The purpose of this construction is to isolate the contribution of screening mistakes: on the good event, \(\widetilde\tau_B\) behaves exactly like \(\widehat\tau_B\), while on the bad event it is forced to agree with the oracle benchmark.

By the triangle inequality,
\begin{align}
\|\widehat\tau_B-\tau_0\|_{L_2(P_0)}
&\le
\underbrace{\|\widehat\tau_B-\widetilde\tau_B\|_{L_2(P_0)}}_{\substack{\text{screening error:}\\\text{difference caused by bad source selection}}}\\
&+
\underbrace{\|\widetilde\tau_B-\tau^\star\|_{L_2(P_0)}}_{\substack{\text{estimation error on the good event:}\\\text{how well the learned predictor approximates the oracle transported target}}}\\
&+
\underbrace{\|\tau^\star-\tau_0\|_{L_2(P_0)}}_{\substack{\text{structural transport bias:}\\\text{irreducible mismatch between the oracle transported target and the true target CATE}}}.
\label{eq:optB_aux}
\end{align}
This is the key decomposition in the proof: it splits total error into a term due to screening failure, a term due to statistical estimation, and a term due to imperfect transport even under an oracle choice of sources.

We next simplify the middle term. Observe that
\[
\widetilde\tau_B-\tau^\star
=
(\widehat\tau_B-\tau^\star)\mathbf 1_E
+
(\tau^\star-\tau^\star)\mathbf 1_{E^c}
=
(\widehat\tau_B-\tau^\star)\mathbf 1_E,
\]
where the \(E^c\) term is zero by construction, because \(\widetilde\tau_B=\tau^\star\) on \(E^c\). Thus, once screening fails, the auxiliary predictor is defined to coincide with the oracle transported target, so there is no estimation discrepancy left on that event. Since multiplication by the indicator \(\mathbf 1_E\) can only reduce the \(L_2(P_0)\) norm, it follows that
\[
\|\widetilde\tau_B-\tau^\star\|_{L_2(P_0)}
=
\|(\widehat\tau_B-\tau^\star)\mathbf 1_E\|_{L_2(P_0)}
\le
\|\widehat\tau_B-\tau^\star\|_{L_2(P_0)}.
\]
This identifies the second term as the usual estimation error relative to the oracle transported target, restricted to the good screening event.

For the first term, define
\[
\Delta_{\mathrm{sel}}
:=
\|\widehat\tau_B-\widetilde\tau_B\|_{L_2(P_0)}
=
\|(\widehat\tau_B-\tau^\star)\mathbf 1_{E^c}\|_{L_2(P_0)}.
\]
This quantity isolates the cost of screening failure: it is nonzero only on the event \(E^c\), that is, only when the placebo-based screening rule selects at least one source outside the oracle transferable set.

Substituting these identities into \eqref{eq:optB_aux} yields
\[
\|\widehat\tau_B-\tau_0\|_{L_2(P_0)}
\le
\|\widehat\tau_B-\tau^\star\|_{L_2(P_0)}
+
\|\tau^\star-\tau_0\|_{L_2(P_0)}
+
\Delta_{\mathrm{sel}},
\]
which proves \eqref{eq:optB_decomp}. The remaining task is therefore to control the structural transport term and the screening term.

To control the structural transport term, define
\[
g_S(x):=\tau_S(x)-\tau_0(x).
\]
This quantity measures, for a given source site \(S\), how far that site's CATE differs from the target CATE at covariate value \(x\). Then, by the definition of \(\tau^\star\),
\[
\tau^\star(x)-\tau_0(x)
=
\mathbb{E}\!\left[g_S(x)\mid S\in\mathcal C_0^\star\right].
\]
In other words, the discrepancy between the oracle transported target and the true target CATE is just the average transport mismatch over the oracle good source set.

Hence
\[
\|\tau^\star-\tau_0\|_{L_2(P_0)}
=
\left\|
\mathbb{E}\!\left[g_S \mid S\in\mathcal C_0^\star\right]
\right\|_{L_2(P_0)}.
\]
Since the \(L_2(P_0)\) norm is convex, by Jensen's inequality (equivalently by Minkowski's integral inequality),
\[
\left\|
\mathbb{E}\!\left[g_S \mid S\in\mathcal C_0^\star\right]
\right\|_{L_2(P_0)}
\le
\mathbb{E}\!\left[
\|g_S\|_{L_2(P_0)}
\mid S\in\mathcal C_0^\star
\right].
\]
This step says that the norm of the average mismatch is at most the average of the individual mismatches. Substituting back \(g_S=\tau_S-\tau_0\), we obtain
\[
\|\tau^\star-\tau_0\|_{L_2(P_0)}
\le
\mathbb{E}\!\left[
\|\tau_S-\tau_0\|_{L_2(P_0)}
\mid S\in\mathcal C_0^\star
\right].
\]
Assumption~\ref{as:A6}(b) therefore implies
\[
\|\tau^\star-\tau_0\|_{L_2(P_0)}\le \epsilon_\tau.
\]
Thus, even under an oracle choice of sources, the remaining target error is controlled by the structural transport tolerance \(\epsilon_\tau\).

Finally, suppose that
\[
\mathbb P(E)\ge 1-\eta,
\qquad\text{so that}\qquad
\mathbb P(E^c)\le \eta.
\]
Here, \(\eta\) is an upper bound on the screening failure probability:
\[
\mathbb P\!\left(\widehat{\mathcal C}_0\not\subseteq \mathcal C_0^\star\right)\le \eta.
\]
Equivalently, the placebo-based screening procedure selects only oracle-good sources with probability at least \(1-\eta\). If, in addition,
\[
M_2
:=
\mathbb E\!\left[\|\widehat\tau_B-\tau^\star\|_{L_2(P_0)}^2\right]
<\infty,
\]
then by Cauchy--Schwarz,
\begin{align*}
\mathbb E[\Delta_{\mathrm{sel}}]
&=
\mathbb E\!\left[
\|\widehat\tau_B-\tau^\star\|_{L_2(P_0)}\mathbf 1_{E^c}
\right]\\
&\le
\Big(
\mathbb E\!\left[\|\widehat\tau_B-\tau^\star\|_{L_2(P_0)}^2\right]
\Big)^{1/2}
\mathbb P(E^c)^{1/2}\\
&\le
M_2^{1/2}\eta^{1/2}.
\end{align*}
This bound has a natural interpretation: the screening penalty can be large only if either the estimation error relative to \(\tau^\star\) is large, or screening failure is common. Under finite second moment and rare failure, the expected screening penalty is therefore small.

Hence, for any fixed \(t>0\), Markov's inequality gives
\[
\mathbb P\!\left(\Delta_{\mathrm{sel}} > t\,\eta^{1/2}\right)
\le
\frac{\mathbb E[\Delta_{\mathrm{sel}}]}{t\,\eta^{1/2}}
\le
\frac{M_2^{1/2}\eta^{1/2}}{t\,\eta^{1/2}}
=
\frac{M_2^{1/2}}{t}.
\]
Since the right-hand side can be made arbitrarily small by taking \(t\) sufficiently large, it follows that
\[
\Delta_{\mathrm{sel}}=O_p(\eta^{1/2}).
\]
That is, the screening-error term is of order \(\eta^{1/2}\) in probability: when the screening failure probability \(\eta\) is small, the contribution of screening mistakes is correspondingly small in stochastic order. This gives the desired probabilistic control on the screening term and completes the proof.
\end{proof}
\clearpage

\section{Experiment Implementation Details}\label{sec:appendix_implementation}

\subsection{Data-Generating Process Details}\label{sec:app_dgp}

The synthetic data generator (\texttt{FairSyntheticRCTGenerator}) creates multi-site RCT data
with the following structure.

\paragraph{Covariate Generation.}
For each site $c \in \{0, 1, \ldots, C\}$ (where $c=0$ is the target),
covariates are generated as:
\[
X_i \mid c \sim \mathcal{N}(\mu_c, I_p),
\]
where $\mu_c$ controls site-level covariate shift. The target site mean is:
\[
\mu_0 = (1-\lambda_{\text{overlap}}) \cdot \bar{\mu}_{\text{source}} + \lambda_{\text{overlap}} \cdot \Delta,
\]
where $\lambda_{\text{overlap}} \in [0,1]$ controls overlap (0 = perfect overlap, 1 = maximal shift),
and $\Delta$ is a random shift vector. Default: $\lambda_{\text{overlap}} = 0.25$ giving
source-vs-target classification AUC $\approx 0.75$.

\paragraph{Treatment Assignment.}
Treatment is randomized via logistic propensity:
\[
A_i \mid X_i \sim \text{Bernoulli}(e(X_i)), \quad e(X_i) = \text{expit}(X_i^\top \gamma_e),
\]
where $\gamma_e$ is sparse with 3 active coefficients. For the target site,
we set $e(X) \equiv p_T$ to achieve the desired $(m_0, m_1)$ split.

\paragraph{Outcome Model.}
Outcomes follow the proxy--deviation decomposition:
\[
Y_i = \mu^{\text{proxy}}_{A_i}(X_i) + \delta_{A_i, c_i}(X_i) + \varepsilon_i, \quad \varepsilon_i \sim \mathcal{N}(0, \sigma^2),
\]
where:
\begin{itemize}
    \item $\mu^{\text{proxy}}_a(x) = x^\top \beta_a + \alpha_a$ is the shared (transferable) component
    \item $\delta_{a,c}(x) = x^\top \gamma_{a,c}$ is the sparse site-specific deviation
    \item $\gamma_{a,c}$ has $\lfloor s \cdot p \rfloor$ nonzero entries (sparsity ratio $s/p$)
\end{itemize}

\paragraph{Nontransfer Scale.}
The nontransfer scale $\sigma_\nu$ controls $\|\gamma_{a,c}\|_2$ relative to $\|\beta_a\|_2$.
Default: $\sigma_\nu = 0.1$, giving signal-to-noise ratio (SNR) $\approx 3$--$4$ for the transferable component.

\paragraph{Nonlinearity Violations (A5 Sensitivity).}
For testing A5 violations, we replace the linear deviation with a mixture:
\[
\delta_{a,c}(x) = (1-\lambda) \cdot x^\top \gamma_{a,c} + \lambda \cdot g(x),
\]
where $g(x) = \sum_{j=1}^p \tanh(x_j)$ is a smooth additive nonlinearity and
$\lambda \in \{0, 0.5, 1\}$ controls the mixture.

\paragraph{Sample Sizes.}
\begin{itemize}
    \item Source sites: $C \in \{2, 5, 10, 20, 50\}$ sites, 1000 samples per site (2000 per site in some sweeps)
    \item Target site: $(m_0, m_1)$ ranges from $(50, 0)$ to $(550, 500)$
    \item Dimensions: $p \in \{10, 20, 50, 100\}$
\end{itemize}

\subsection{Benchmark Method Implementations}\label{sec:app_methods}

All methods receive the same data and are evaluated on held-out target samples.
We describe each method category below.

\paragraph{Baseline Methods.}

\textit{TargetOnly.}
Standard doubly robust (DR) learner using only target data. Fits outcome models
$\hat{\mu}_0, \hat{\mu}_1$ on target placebo and treated samples, then computes
DR pseudo-outcomes:
\[
\tilde{\tau}_i = \frac{A_i(Y_i - \hat{\mu}_1(X_i))}{e(X_i)}
              - \frac{(1-A_i)(Y_i - \hat{\mu}_0(X_i))}{1-e(X_i)}
              + \hat{\mu}_1(X_i) - \hat{\mu}_0(X_i).
\]
Requires target treated data ($m_1 > 0$). Uses Ridge regression for $\hat{\mu}_a$.

\textit{ProxyOnly.}
Pools all source data to fit outcome models $\hat{\mu}^{\text{proxy}}_a(x)$ using Random Forest,
then predicts CATE as $\hat{\tau}(x) = \hat{\mu}^{\text{proxy}}_1(x) - \hat{\mu}^{\text{proxy}}_0(x)$.
Does not use any target data for calibration---serves as a transfer-without-anchoring baseline.

\paragraph{Transport Baselines.}

\textit{IPW-Transport.}
Inverse probability weighting estimator for transportability \citep{westreich2017transportability}.
Fits a selection model $\hat{s}(X) = P(\text{source} \mid X)$ via logistic regression,
then reweights source outcomes:
\[
\hat{\tau}^{\text{IPW}} = \frac{1}{n_1} \sum_{i: A_i=1} \frac{Y_i \cdot \hat{w}(X_i)}{\sum_j \hat{w}(X_j)}
                        - \frac{1}{n_0} \sum_{i: A_i=0} \frac{Y_i \cdot \hat{w}(X_i)}{\sum_j \hat{w}(X_j)},
\]
where $\hat{w}(x) = (1-\hat{s}(x))/\hat{s}(x)$ are the odds weights.

\textit{OM-Transport.}
Outcome model transport estimator \citep{dahabreh2019generalizing}.
Fits arm-specific outcome models $\hat{g}_a(x)$ on source data,
then averages predictions over target covariates:
\[
\hat{\tau}^{\text{OM}} = \frac{1}{m_0+m_1} \sum_{i \in \text{target}} \left[\hat{g}_1(X_i) - \hat{g}_0(X_i)\right].
\]
Uses Random Forest for $\hat{g}_a$.

\textit{EntropyBal.}
Entropy balancing for covariate shift adjustment \citep{hainmueller2012entropy}.
Finds weights $w_i$ that minimize $\sum_i w_i \log(w_i)$ subject to moment constraints:
\[
\sum_{i \in \text{source}} w_i X_i = \bar{X}_{\text{target}}.
\]
Then applies weighted outcome regression.

\paragraph{Proposed Methods (glmtrans-based).}

Our proposed methods use the \texttt{glmtrans} R package \citep{tian2023transfer} for
$\ell_1$-penalized transfer learning with automatic source detection.

\textit{Proposed.}
Two-step transfer learning:
\begin{enumerate}
    \item \textbf{Source detection}: For each source $k$, test $H_0: \beta^{(k)} = \beta^{(0)}$
          using truncated $\ell_1$ estimator. Select transferable sources $\mathcal{A}$.
    \item \textbf{Transfer estimation}: Solve
    \[
    \hat{\beta}_a = \argmin_\beta \frac{1}{2m}\|Y_a - X_a\beta\|_2^2
                  + \lambda_1 \|\beta - \bar{\beta}_{\mathcal{A}}\|_1 + \lambda_2 \|\beta\|_1,
    \]
    where $\bar{\beta}_{\mathcal{A}}$ is the pooled source estimate from selected sources.
\end{enumerate}
CATE is computed as plug-in: $\hat{\tau}(x) = x^\top(\hat{\beta}_1 - \hat{\beta}_0)$.

\textit{Proposed-CF.}
Combines glmtrans with cross-fitted DR pseudo-outcomes:
\begin{enumerate}
    \item Split target data into $K=2$ folds
    \item For each fold $k$: fit $\hat{\mu}_0^{(-k)}, \hat{\mu}_1^{(-k)}$ on remaining folds
          (propensity $e(X)$ is known by randomization design, not estimated)
    \item Compute DR pseudo-outcomes $\tilde{\tau}_i$ for fold $k$ samples
    \item Run glmtrans on pseudo-outcomes
\end{enumerate}

\textit{Proposed-B.}
For disconnected targets ($m_1 = 0$):
\begin{enumerate}
    \item Use target placebo outcomes to run glmtrans source detection on the control arm,
          identifying transferable sources $\mathcal{A}$
    \item Fit source-side DR CATE using only selected source data
    \item Transport the source CATE estimate to the target covariate distribution
          by averaging over target placebo covariates
\end{enumerate}

\subsection{Hyperparameters and Tuning}\label{sec:app_hyperparams}

\paragraph{Regularization.}
\begin{itemize}
    \item LASSO/Ridge: 5-fold cross-validation with \texttt{LassoCV}/\texttt{RidgeCV}
    \item glmtrans: $\lambda$ selected by 5-fold cross-validation minimizing MSE
    \item Random Forest (proxy outcome models): 100 trees, max depth 8, min samples leaf 5
    \item Random Forest (CATE meta-learner): 100 trees, max depth 5, min samples leaf 5
\end{itemize}

\paragraph{Cross-fitting.}
Proposed-CF uses $K=2$ stratified folds (stratified by treatment arm).
TargetOnly and AnchorOnly use $K=5$ stratified folds.

\paragraph{Propensity Scores.}
For the proposed methods, propensity scores $e(X)$ are the known randomization probabilities
and are not estimated. Transport baselines that require propensity estimation clip scores to
$[\epsilon, 1-\epsilon]$ with $\epsilon = 0.01$ to avoid extreme weights.

\paragraph{Reproducibility.}
All experiments use fixed random seeds. Each scenario runs $R=100$ Monte Carlo replicates.
Code is available at \url{https://github.com/wangzilongri/SparseTLDRICHI2026}.

\clearpage

\captionsetup{type=table,position=above,font=normalsize}
\setcounter{section}{3}
\noindent\textbf{\large\thesection.\ Supplementary Tables}\label{sec:appendix_tables}\par\vspace{0.5em}
The following tables report full results with error bars (mean $\pm$ SD over 100 MC replicates)
for all metrics across all experiments. Method abbreviations: Target=TargetOnly, Proxy=ProxyOnly,
IPW=IPW-Transport, OM=OM-Transport, EB=EntropyBal, Anch=AnchorOnly,
Prop=Proposed, Prop-CF=Proposed-CF, Prop-B=Proposed-B.

\subsection{Target Budget vs.\ Dimensionality}\label{sec:app_tables_dim}
\begin{table}[H]
\centering
\caption{PEHE (mean$\pm$SD) across budget and $p$. Lower is better.}
\label{tab:dim_sweep_pehe_full}
\scriptsize
\setlength{\tabcolsep}{2pt}
\begin{tabular}{clccccc}
\toprule
$p$ & Method & (50,0) & (100,50) & (150,100) & (250,200) & (550,500) \\
\midrule
10 & TargetOnly & -- & 1.69$\pm$0.49 & 1.43$\pm$0.39 & 1.42$\pm$0.38 & 1.38$\pm$0.45 \\
 & ProxyOnly & 2.34$\pm$0.67 & 2.09$\pm$0.66 & 1.96$\pm$0.66 & 2.08$\pm$0.68 & 2.10$\pm$0.83 \\
 & IPW-Transport & 1.80$\pm$0.83 & 1.80$\pm$0.75 & 1.60$\pm$0.70 & 1.78$\pm$0.77 & 1.82$\pm$1.06 \\
 & OM-Transport & 1.80$\pm$0.83 & 1.81$\pm$0.75 & 1.60$\pm$0.70 & 1.78$\pm$0.77 & 1.81$\pm$1.06 \\
 & EntropyBal & 1.81$\pm$0.83 & 1.81$\pm$0.75 & 1.60$\pm$0.70 & 1.78$\pm$0.77 & 1.82$\pm$1.06 \\
 & AnchorOnly & -- & 1.71$\pm$0.53 & 1.46$\pm$0.41 & 1.44$\pm$0.41 & 1.39$\pm$0.46 \\
 & Proposed & -- & \textbf{0.53$\pm$0.14} & 0.43$\pm$0.16 & 0.35$\pm$0.15 & 0.28$\pm$0.14 \\
\bottomrule
\end{tabular}
\end{table}
\begin{table}[H]
\centering
\caption{ATE (mean$\pm$SD) across budget and $p$. Lower is better.}
\label{tab:dim_sweep_ate_full}
\scriptsize
\setlength{\tabcolsep}{2pt}
\begin{tabular}{clccccc}
\toprule
$p$ & Method & (50,0) & (100,50) & (150,100) & (250,200) & (550,500) \\
\midrule
10 & TargetOnly & -- & 0.19$\pm$0.15 & 0.14$\pm$0.10 & 0.10$\pm$0.07 & 0.07$\pm$0.06 \\
 & ProxyOnly & 0.69$\pm$0.52 & 0.56$\pm$0.45 & 0.53$\pm$0.48 & 0.69$\pm$0.48 & 0.67$\pm$0.48 \\
 & IPW-Transport & 0.77$\pm$0.56 & 0.74$\pm$0.50 & 0.52$\pm$0.42 & 0.64$\pm$0.50 & 0.70$\pm$0.57 \\
 & OM-Transport & 0.77$\pm$0.56 & 0.74$\pm$0.50 & 0.52$\pm$0.42 & 0.64$\pm$0.49 & 0.70$\pm$0.57 \\
 & EntropyBal & 0.77$\pm$0.56 & 0.74$\pm$0.51 & 0.52$\pm$0.42 & 0.64$\pm$0.50 & 0.70$\pm$0.57 \\
 & AnchorOnly & -- & 0.21$\pm$0.16 & 0.13$\pm$0.11 & 0.10$\pm$0.09 & 0.08$\pm$0.06 \\
 & Proposed & -- & \textbf{0.08$\pm$0.06} & \textbf{0.05$\pm$0.04} & 0.04$\pm$0.03 & \textbf{0.03$\pm$0.02} \\
\bottomrule
\end{tabular}
\end{table}
\begin{table}[H]
\centering
\caption{Spearman (mean$\pm$SD) across budget and $p$. Higher is better.}
\label{tab:dim_sweep_spearman_full}
\scriptsize
\setlength{\tabcolsep}{2pt}
\begin{tabular}{clccccc}
\toprule
$p$ & Method & (50,0) & (100,50) & (150,100) & (250,200) & (550,500) \\
\midrule
10 & TargetOnly & -- & 0.80$\pm$0.06 & 0.85$\pm$0.04 & 0.86$\pm$0.03 & 0.88$\pm$0.03 \\
 & ProxyOnly & 0.64$\pm$0.15 & 0.68$\pm$0.16 & 0.72$\pm$0.15 & 0.71$\pm$0.14 & 0.71$\pm$0.20 \\
 & IPW-Transport & 0.81$\pm$0.17 & 0.77$\pm$0.21 & 0.80$\pm$0.19 & 0.78$\pm$0.19 & 0.78$\pm$0.25 \\
 & OM-Transport & 0.81$\pm$0.17 & 0.77$\pm$0.21 & 0.80$\pm$0.19 & 0.78$\pm$0.19 & 0.78$\pm$0.25 \\
 & EntropyBal & 0.81$\pm$0.17 & 0.77$\pm$0.21 & 0.80$\pm$0.19 & 0.78$\pm$0.19 & 0.78$\pm$0.25 \\
 & AnchorOnly & -- & 0.80$\pm$0.07 & 0.85$\pm$0.05 & 0.86$\pm$0.04 & 0.87$\pm$0.03 \\
 & Proposed & -- & \textbf{0.98$\pm$0.02} & 0.98$\pm$0.02 & 0.99$\pm$0.01 & 0.99$\pm$0.01 \\
\bottomrule
\end{tabular}
\end{table}
\begin{table}[H]
\centering
\caption{Regret (mean$\pm$SD) across budget and $p$. Lower is better.}
\label{tab:dim_sweep_regret_full}
\scriptsize
\setlength{\tabcolsep}{2pt}
\begin{tabular}{clccccc}
\toprule
$p$ & Method & (50,0) & (100,50) & (150,100) & (250,200) & (550,500) \\
\midrule
10 & TargetOnly & -- & 0.20$\pm$0.09 & 0.14$\pm$0.05 & 0.13$\pm$0.05 & 0.12$\pm$0.05 \\
 & ProxyOnly & 0.43$\pm$0.25 & 0.36$\pm$0.22 & 0.32$\pm$0.22 & 0.35$\pm$0.21 & 0.35$\pm$0.27 \\
 & IPW-Transport & 0.27$\pm$0.25 & 0.28$\pm$0.24 & 0.22$\pm$0.20 & 0.27$\pm$0.24 & 0.29$\pm$0.37 \\
 & OM-Transport & 0.27$\pm$0.25 & 0.28$\pm$0.24 & 0.22$\pm$0.20 & 0.27$\pm$0.25 & 0.29$\pm$0.37 \\
 & EntropyBal & 0.27$\pm$0.24 & 0.28$\pm$0.25 & 0.22$\pm$0.20 & 0.27$\pm$0.24 & 0.29$\pm$0.37 \\
 & AnchorOnly & -- & 0.19$\pm$0.08 & 0.14$\pm$0.06 & 0.13$\pm$0.06 & 0.12$\pm$0.06 \\
 & Proposed & -- & \textbf{0.02$\pm$0.01} & 0.01$\pm$0.01 & 0.01$\pm$0.01 & 0.01$\pm$0.01 \\
\bottomrule
\end{tabular}
\end{table}
\FloatBarrier

\subsection{Source Site Scaling}\label{sec:app_tables_sources}
\begin{table}[H]
\centering
\caption{PEHE (mean$\pm$SD) across sources $C$ and budget. Lower is better.}
\label{tab:sources_sweep_pehe_full}
\scriptsize
\setlength{\tabcolsep}{2pt}
\begin{tabular}{clccccc}
\toprule
$C$ & Method & (50,0) & (100,50) & (150,100) & (250,200) & (550,500) \\
\midrule
2 & TargetOnly & -- & 5.33$\pm$0.92 & 4.95$\pm$1.18 & 4.70$\pm$0.75 & 4.46$\pm$0.91 \\
 & ProxyOnly & 5.67$\pm$1.05 & 5.62$\pm$1.06 & 5.55$\pm$1.37 & 5.42$\pm$1.16 & 5.19$\pm$1.22 \\
 & IPW-Transport & 3.64$\pm$1.45 & 4.17$\pm$1.88 & 3.78$\pm$1.85 & 4.03$\pm$1.94 & 3.74$\pm$1.97 \\
 & OM-Transport & 3.54$\pm$1.45 & 4.01$\pm$1.93 & 3.59$\pm$1.89 & 3.86$\pm$1.93 & 3.52$\pm$1.98 \\
 & EntropyBal & 3.76$\pm$1.45 & 4.23$\pm$1.87 & 3.79$\pm$1.85 & 4.04$\pm$1.94 & 3.73$\pm$1.98 \\
 & AnchorOnly & -- & 5.39$\pm$1.00 & 5.01$\pm$1.26 & 4.75$\pm$0.78 & 4.52$\pm$0.94 \\
 & Proposed & -- & \textbf{1.76$\pm$0.72} & \textbf{0.84$\pm$0.16} & 0.62$\pm$0.12 & 0.45$\pm$0.11 \\
\bottomrule
\end{tabular}
\end{table}
\begin{table}[H]
\centering
\caption{ATE (mean$\pm$SD) across sources $C$ and budget. Lower is better.}
\label{tab:sources_sweep_ate_full}
\scriptsize
\setlength{\tabcolsep}{2pt}
\begin{tabular}{clccccc}
\toprule
$C$ & Method & (50,0) & (100,50) & (150,100) & (250,200) & (550,500) \\
\midrule
2 & TargetOnly & -- & 0.49$\pm$0.41 & 0.40$\pm$0.35 & 0.31$\pm$0.24 & 0.21$\pm$0.15 \\
 & ProxyOnly & 1.29$\pm$1.05 & 1.16$\pm$0.97 & 1.08$\pm$0.96 & 1.20$\pm$1.20 & 1.13$\pm$0.94 \\
 & IPW-Transport & 0.79$\pm$0.63 & 1.13$\pm$0.95 & 0.88$\pm$0.64 & 1.12$\pm$0.97 & 0.79$\pm$0.66 \\
 & OM-Transport & 0.79$\pm$0.64 & 1.11$\pm$0.93 & 0.88$\pm$0.66 & 1.12$\pm$0.97 & 0.80$\pm$0.67 \\
 & EntropyBal & 0.82$\pm$0.63 & 1.14$\pm$0.96 & 0.88$\pm$0.64 & 1.13$\pm$0.97 & 0.79$\pm$0.66 \\
 & AnchorOnly & -- & 0.54$\pm$0.46 & 0.38$\pm$0.33 & 0.30$\pm$0.23 & 0.21$\pm$0.14 \\
 & Proposed & -- & \textbf{0.21$\pm$0.17} & \textbf{0.08$\pm$0.07} & \textbf{0.05$\pm$0.03} & \textbf{0.03$\pm$0.03} \\
\bottomrule
\end{tabular}
\end{table}
\begin{table}[H]
\centering
\caption{Spearman (mean$\pm$SD) across sources $C$ and budget. Higher is better.}
\label{tab:sources_sweep_spearman_full}
\scriptsize
\setlength{\tabcolsep}{2pt}
\begin{tabular}{clccccc}
\toprule
$C$ & Method & (50,0) & (100,50) & (150,100) & (250,200) & (550,500) \\
\midrule
2 & TargetOnly & -- & 0.47$\pm$0.08 & 0.60$\pm$0.06 & 0.66$\pm$0.05 & 0.68$\pm$0.04 \\
 & ProxyOnly & 0.34$\pm$0.09 & 0.41$\pm$0.10 & 0.45$\pm$0.10 & 0.49$\pm$0.10 & 0.49$\pm$0.12 \\
 & IPW-Transport & 0.78$\pm$0.15 & 0.74$\pm$0.19 & 0.78$\pm$0.16 & 0.75$\pm$0.19 & 0.77$\pm$0.20 \\
 & OM-Transport & 0.79$\pm$0.15 & 0.75$\pm$0.19 & 0.80$\pm$0.17 & 0.77$\pm$0.19 & 0.79$\pm$0.20 \\
 & EntropyBal & 0.77$\pm$0.15 & 0.73$\pm$0.19 & 0.78$\pm$0.16 & 0.75$\pm$0.19 & 0.77$\pm$0.20 \\
 & AnchorOnly & -- & 0.46$\pm$0.08 & 0.59$\pm$0.06 & 0.65$\pm$0.04 & 0.67$\pm$0.05 \\
 & Proposed & -- & \textbf{0.94$\pm$0.05} & \textbf{0.99$\pm$0.00} & \textbf{0.99$\pm$0.00} & 1.00$\pm$0.00 \\
\bottomrule
\end{tabular}
\end{table}
\begin{table}[H]
\centering
\caption{Regret (mean$\pm$SD) across sources $C$ and budget. Lower is better.}
\label{tab:sources_sweep_regret_full}
\scriptsize
\setlength{\tabcolsep}{2pt}
\begin{tabular}{clccccc}
\toprule
$C$ & Method & (50,0) & (100,50) & (150,100) & (250,200) & (550,500) \\
\midrule
2 & TargetOnly & -- & 1.18$\pm$0.24 & 0.92$\pm$0.28 & 0.79$\pm$0.18 & 0.71$\pm$0.19 \\
 & ProxyOnly & 1.58$\pm$0.51 & 1.50$\pm$0.45 & 1.39$\pm$0.56 & 1.34$\pm$0.60 & 1.23$\pm$0.50 \\
 & IPW-Transport & 0.50$\pm$0.42 & 0.67$\pm$0.55 & 0.55$\pm$0.56 & 0.63$\pm$0.56 & 0.56$\pm$0.57 \\
 & OM-Transport & 0.48$\pm$0.41 & 0.63$\pm$0.55 & 0.51$\pm$0.55 & 0.60$\pm$0.57 & 0.52$\pm$0.59 \\
 & EntropyBal & 0.52$\pm$0.42 & 0.68$\pm$0.54 & 0.55$\pm$0.56 & 0.63$\pm$0.56 & 0.55$\pm$0.57 \\
 & AnchorOnly & -- & 1.20$\pm$0.30 & 0.93$\pm$0.28 & 0.82$\pm$0.17 & 0.75$\pm$0.21 \\
 & Proposed & -- & \textbf{0.12$\pm$0.10} & \textbf{0.02$\pm$0.01} & \textbf{0.01$\pm$0.01} & 0.01$\pm$0.00 \\
\bottomrule
\end{tabular}
\end{table}
\FloatBarrier

\subsection{Sensitivity to A5 Violations}\label{sec:app_tables_a5}
\begin{table}[H]
\centering
\caption{PEHE (mean$\pm$SD) under A5 violations. Lower is better.}
\label{tab:a5_violation_pehe_full}
\resizebox{\columnwidth}{!}{%
\scriptsize
\begin{tabular}{llccc}
\toprule
$s/p$ & Method & $\Gamma=0.0$ & $\Gamma=0.5$ & $\Gamma=1.0$ \\
\midrule
0.05 & TargetOnly & 4.41$\pm$0.70 & 4.25$\pm$0.53 & 4.30$\pm$0.57 \\
 & ProxyOnly & 4.97$\pm$0.89 & 4.72$\pm$0.68 & 4.84$\pm$0.72 \\
 & IPW-Transport & 2.79$\pm$1.25 & 2.13$\pm$1.02 & 2.20$\pm$1.11 \\
 & OM-Transport & 2.67$\pm$1.26 & 2.04$\pm$1.03 & 2.10$\pm$1.13 \\
 & EntropyBal & 2.80$\pm$1.25 & 2.14$\pm$1.02 & 2.21$\pm$1.11 \\
 & AnchorOnly & 4.45$\pm$0.70 & 4.28$\pm$0.53 & 4.34$\pm$0.56 \\
 & Proposed & 0.44$\pm$0.11 & \textbf{1.16$\pm$0.68} & \textbf{1.92$\pm$1.19} \\
\bottomrule
\end{tabular}%
}
\end{table}
\begin{table}[H]
\centering
\caption{ATE (mean$\pm$SD) under A5 violations. Lower is better.}
\label{tab:a5_violation_ate_full}
\resizebox{\columnwidth}{!}{%
\scriptsize
\begin{tabular}{llccc}
\toprule
$s/p$ & Method & $\Gamma=0.0$ & $\Gamma=0.5$ & $\Gamma=1.0$ \\
\midrule
0.05 & TargetOnly & 0.22$\pm$0.14 & 0.21$\pm$0.17 & 0.16$\pm$0.12 \\
 & ProxyOnly & 0.92$\pm$0.75 & 0.91$\pm$0.74 & 0.98$\pm$0.79 \\
 & IPW-Transport & 0.80$\pm$0.60 & 0.65$\pm$0.51 & 0.61$\pm$0.47 \\
 & OM-Transport & 0.79$\pm$0.58 & 0.65$\pm$0.50 & 0.60$\pm$0.47 \\
 & EntropyBal & 0.80$\pm$0.60 & 0.65$\pm$0.51 & 0.61$\pm$0.46 \\
 & AnchorOnly & 0.20$\pm$0.14 & 0.19$\pm$0.14 & 0.15$\pm$0.12 \\
 & Proposed & 0.04$\pm$0.03 & \textbf{0.07$\pm$0.07} & \textbf{0.09$\pm$0.07} \\
\bottomrule
\end{tabular}%
}
\end{table}
\begin{table}[H]
\centering
\caption{Spearman (mean$\pm$SD) under A5 violations. Higher is better.}
\label{tab:a5_violation_spearman_full}
\resizebox{\columnwidth}{!}{%
\scriptsize
\begin{tabular}{llccc}
\toprule
$s/p$ & Method & $\Gamma=0.0$ & $\Gamma=0.5$ & $\Gamma=1.0$ \\
\midrule
0.05 & TargetOnly & 0.68$\pm$0.04 & 0.67$\pm$0.04 & 0.67$\pm$0.04 \\
 & ProxyOnly & 0.54$\pm$0.09 & 0.55$\pm$0.08 & 0.55$\pm$0.08 \\
 & IPW-Transport & 0.87$\pm$0.11 & 0.91$\pm$0.08 & 0.91$\pm$0.08 \\
 & OM-Transport & 0.88$\pm$0.11 & 0.92$\pm$0.08 & 0.92$\pm$0.08 \\
 & EntropyBal & 0.87$\pm$0.11 & 0.91$\pm$0.08 & 0.91$\pm$0.08 \\
 & AnchorOnly & 0.67$\pm$0.04 & 0.66$\pm$0.04 & 0.66$\pm$0.04 \\
 & Proposed & 1.00$\pm$0.00 & \textbf{0.97$\pm$0.03} & \textbf{0.92$\pm$0.08} \\
\bottomrule
\end{tabular}%
}
\end{table}
\begin{table}[H]
\centering
\caption{Regret (mean$\pm$SD) under A5 violations. Lower is better.}
\label{tab:a5_violation_regret_full}
\resizebox{\columnwidth}{!}{%
\scriptsize
\begin{tabular}{llccc}
\toprule
$s/p$ & Method & $\Gamma=0.0$ & $\Gamma=0.5$ & $\Gamma=1.0$ \\
\midrule
0.05 & TargetOnly & 0.71$\pm$0.15 & 0.68$\pm$0.14 & 0.70$\pm$0.15 \\
 & ProxyOnly & 1.11$\pm$0.40 & 1.02$\pm$0.33 & 1.08$\pm$0.31 \\
 & IPW-Transport & 0.32$\pm$0.29 & 0.20$\pm$0.19 & 0.22$\pm$0.22 \\
 & OM-Transport & 0.30$\pm$0.28 & 0.19$\pm$0.19 & 0.20$\pm$0.22 \\
 & EntropyBal & 0.32$\pm$0.29 & 0.20$\pm$0.20 & 0.22$\pm$0.22 \\
 & AnchorOnly & 0.73$\pm$0.16 & 0.71$\pm$0.14 & 0.72$\pm$0.15 \\
 & Proposed & 0.01$\pm$0.00 & \textbf{0.06$\pm$0.07} & \textbf{0.18$\pm$0.21} \\
\bottomrule
\end{tabular}%
}
\end{table}
\FloatBarrier

\subsection{CATE Calibration}\label{sec:app_tables_calib}
The following tables report CATE calibration metrics: Slope (ideal value = 1.0), R$^2$ (higher is better), and Expected Calibration Error (ECE, lower is better). Best value in each scenario is bolded.

\begin{table}[H]
\centering
\caption{CATE Calibration: Dim Sweep at budget (150,100). Slope (ideal=1), R$^2$ (higher=better), ECE (lower=better).}
\label{tab:dim_calib_full}
\resizebox{\columnwidth}{!}{%
\scriptsize
\setlength{\tabcolsep}{2pt}
\begin{tabular}{@{}l|ccc|ccc|ccc|ccc@{}}
\toprule
 & \multicolumn{3}{c|}{p=10} & \multicolumn{3}{c|}{p=20} & \multicolumn{3}{c|}{p=50} & \multicolumn{3}{c}{p=100} \\
Method & Slope & R$^2$ & ECE & Slope & R$^2$ & ECE & Slope & R$^2$ & ECE & Slope & R$^2$ & ECE \\
\midrule
TargetOnly & 1.18 & 0.72 & 0.40 & 1.38 & 0.57 & 0.76 & 1.56 & 0.37 & 1.21 & 1.45 & 0.20 & 1.47 \\
ProxyOnly & 1.13 & 0.54 & 0.73 & 1.36 & 0.40 & 1.03 & 1.57 & 0.24 & 1.52 & 1.53 & 0.13 & 1.79 \\
IPW-Transport & 0.94 & 0.69 & 0.65 & 0.95 & 0.70 & 0.87 & 0.96 & 0.75 & 0.99 & 0.96 & 0.75 & 1.18 \\
OM-Transport & 0.94 & 0.69 & 0.66 & 0.95 & 0.70 & 0.87 & 0.97 & 0.75 & 0.98 & 0.97 & 0.76 & 1.14 \\
EntropyBal & 0.94 & 0.69 & 0.66 & 0.94 & 0.70 & 0.87 & 0.95 & 0.74 & 0.99 & 0.91 & 0.71 & 1.30 \\
AnchorOnly & 1.16 & 0.71 & 0.42 & 1.33 & 0.55 & 0.79 & 1.61 & 0.36 & 1.31 & 1.45 & 0.19 & 1.57 \\
Proposed & 1.03 & \textbf{0.97} & 0.11 & 1.04 & \textbf{0.98} & 0.15 & 1.02 & \textbf{0.98} & \textbf{0.15} & \textbf{1.02} & \textbf{0.96} & \textbf{0.27} \\
Proposed-CF & \textbf{1.00} & \textbf{0.97} & \textbf{0.09} & \textbf{1.00} & \textbf{0.98} & \textbf{0.11} & 0.98 & 0.96 & 0.22 & 0.97 & 0.85 & 0.55 \\
Proposed-B & 0.95 & 0.72 & 0.62 & 0.94 & 0.70 & 0.90 & \textbf{1.00} & 0.82 & 0.91 & 0.98 & 0.76 & 1.14 \\
\bottomrule
\end{tabular}%
}
\end{table}
\vspace{1.5em}
\begin{table}[H]
\centering
\caption{CATE Calibration: Sources Sweep at budget (150,100). Slope (ideal=1), R$^2$ (higher=better), ECE (lower=better).}
\label{tab:sources_calib_full}
\resizebox{\columnwidth}{!}{%
\scriptsize
\setlength{\tabcolsep}{2pt}
\begin{tabular}{@{}l|ccc|ccc|ccc|ccc|ccc@{}}
\toprule
 & \multicolumn{3}{c|}{C=2} & \multicolumn{3}{c|}{C=5} & \multicolumn{3}{c|}{C=10} & \multicolumn{3}{c|}{C=20} & \multicolumn{3}{c}{C=50} \\
Method & Slope & R$^2$ & ECE & Slope & R$^2$ & ECE & Slope & R$^2$ & ECE & Slope & R$^2$ & ECE & Slope & R$^2$ & ECE \\
\midrule
TargetOnly & 1.48 & 0.36 & 1.24 & 1.50 & 0.35 & 1.20 & 1.48 & 0.36 & 1.20 & 1.54 & 0.35 & 1.23 & 1.50 & 0.35 & 1.20 \\
ProxyOnly & 1.43 & 0.22 & 1.48 & 1.47 & 0.23 & 1.34 & 1.57 & 0.24 & 1.54 & 1.50 & 0.22 & 1.44 & 1.52 & 0.22 & 1.54 \\
IPW-Transport & 0.84 & 0.66 & 1.32 & 0.88 & 0.69 & 1.18 & 0.97 & 0.71 & 1.14 & 0.97 & 0.76 & 0.86 & \textbf{0.99} & 0.74 & 0.96 \\
OM-Transport & 0.87 & 0.69 & 1.25 & 0.92 & 0.71 & 1.09 & \textbf{0.99} & 0.73 & 1.11 & 0.98 & 0.77 & 0.84 & \textbf{0.99} & 0.74 & 0.95 \\
EntropyBal & 0.84 & 0.66 & 1.32 & 0.87 & 0.68 & 1.20 & 0.95 & 0.70 & 1.15 & 0.95 & 0.75 & 0.88 & \textbf{0.99} & 0.73 & 0.97 \\
AnchorOnly & 1.46 & 0.34 & 1.27 & 1.50 & 0.34 & 1.27 & 1.51 & 0.34 & 1.31 & 1.53 & 0.33 & 1.26 & 1.54 & 0.34 & 1.31 \\
Proposed & 1.03 & \textbf{0.98} & \textbf{0.19} & 1.03 & \textbf{0.98} & \textbf{0.17} & 1.02 & \textbf{0.99} & \textbf{0.16} & 1.02 & \textbf{0.98} & \textbf{0.15} & 1.02 & \textbf{0.99} & \textbf{0.15} \\
Proposed-CF & \textbf{0.98} & 0.95 & 0.29 & \textbf{0.99} & 0.96 & 0.23 & 0.98 & 0.96 & 0.24 & 0.99 & 0.96 & 0.19 & 0.98 & 0.96 & 0.23 \\
Proposed-B & 0.88 & 0.70 & 1.27 & 0.94 & 0.78 & 0.99 & 1.01 & 0.80 & 1.03 & \textbf{1.00} & 0.83 & 0.78 & 1.04 & 0.87 & 0.80 \\
\bottomrule
\end{tabular}%
}
\end{table}
\vspace{1.5em}
\begin{table}[H]
\centering
\caption{CATE Calibration under A5 Violations. Slope (ideal=1), R$^2$ (higher=better), ECE (lower=better).}
\label{tab:a5_calib_full}
\resizebox{\columnwidth}{!}{%
\scriptsize
\setlength{\tabcolsep}{2pt}
\begin{tabular}{@{}cl|ccc|ccc|ccc@{}}
\toprule
 & & \multicolumn{3}{c|}{$\Gamma$=0} & \multicolumn{3}{c|}{$\Gamma$=0.5} & \multicolumn{3}{c}{$\Gamma$=1} \\
$s/p$ & Method & Slope & R$^2$ & ECE & Slope & R$^2$ & ECE & Slope & R$^2$ & ECE \\
\midrule
0.05 & TargetOnly & 1.69 & 0.48 & 1.33 & 1.75 & 0.47 & 1.32 & 1.68 & 0.47 & 1.26 \\
 & ProxyOnly & 1.51 & 0.32 & 1.37 & 1.54 & 0.33 & 1.36 & 1.53 & 0.32 & 1.45 \\
 & IPW-Transport & 0.95 & 0.78 & 0.92 & 0.97 & 0.85 & 0.71 & 0.98 & 0.85 & 0.64 \\
 & OM-Transport & 0.97 & 0.80 & 0.89 & 0.98 & 0.86 & 0.69 & 0.99 & \textbf{0.86} & 0.62 \\
 & EntropyBal & 0.95 & 0.78 & 0.93 & 0.97 & 0.85 & 0.71 & 0.98 & 0.85 & 0.64 \\
 & AnchorOnly & 1.71 & 0.47 & 1.34 & 1.74 & 0.46 & 1.30 & 1.68 & 0.45 & 1.25 \\
\bottomrule
\end{tabular}%
}
\end{table}
\FloatBarrier

\subsection{IHDP Semi-Synthetic (Full Metrics)}\label{sec:app_tables_ihdp}
Full metrics (PEHE, ATE error, Spearman, Regret, ECE) for the IHDP connected and disconnected regimes (mean $\pm$ SD over 50 realizations).

\begin{table}[H]
\centering
\caption{IHDP connected regime: full metrics (PEHE, ATE error, Spearman, Policy regret, ECE).}
\label{tab:ihdp_connected_appendix}
\adjustbox{max width=\columnwidth, max totalheight=0.95\textheight, keepaspectratio}{%
\scriptsize
\setlength{\tabcolsep}{3pt}
\begin{tabular}{lrrrrrr}
\toprule
Method & Budget & PEHE & ATE err & Spearman & Regret & ECE \\
\midrule
TargetOnly & (25,25) & 4.22 $\pm$ 6.07 & 0.73 $\pm$ 1.12 & 0.58 $\pm$ 0.24 & 0.39 $\pm$ 0.81 & 1.92 $\pm$ 2.23 \\
TargetOnly & (25,50) & 5.45 $\pm$ 7.12 & 0.91 $\pm$ 1.04 & 0.60 $\pm$ 0.20 & 0.79 $\pm$ 1.29 & 2.09 $\pm$ 2.09 \\
TargetOnly & (25,100) & 3.09 $\pm$ 3.97 & 0.67 $\pm$ 1.11 & 0.53 $\pm$ 0.23 & 0.36 $\pm$ 1.01 & 1.30 $\pm$ 1.34 \\
TargetOnly & (50,25) & 3.39 $\pm$ 5.17 & 0.37 $\pm$ 0.52 & 0.71 $\pm$ 0.21 & 0.19 $\pm$ 0.34 & 1.41 $\pm$ 1.69 \\
TargetOnly & (50,50) & 3.63 $\pm$ 5.43 & 0.41 $\pm$ 0.66 & 0.66 $\pm$ 0.22 & 0.23 $\pm$ 0.46 & 1.23 $\pm$ 1.43 \\
TargetOnly & (50,100) & 4.58 $\pm$ 6.19 & 0.54 $\pm$ 0.76 & 0.71 $\pm$ 0.21 & 0.34 $\pm$ 0.53 & 1.61 $\pm$ 1.95 \\
TargetOnly & (100,25) & 2.48 $\pm$ 3.33 & 0.29 $\pm$ 0.35 & 0.80 $\pm$ 0.21 & 0.09 $\pm$ 0.14 & 0.97 $\pm$ 1.03 \\
\bottomrule
\end{tabular}%
}
\end{table}
\FloatBarrier
\begin{table}[H]
\centering
\caption{IHDP disconnected regime ($m_1=0$): full metrics.}
\label{tab:ihdp_disconnected_appendix}
\adjustbox{max width=\columnwidth, max totalheight=0.95\textheight, keepaspectratio}{%
\scriptsize
\setlength{\tabcolsep}{3pt}
\begin{tabular}{lrrrrrr}
\toprule
Method & Budget & PEHE & ATE err & Spearman & Regret & ECE \\
\midrule
ProxyOnly & 25 & 2.64 $\pm$ 2.65 & 1.06 $\pm$ 1.61 & 0.65 $\pm$ 0.25 & 0.16 $\pm$ 0.27 & 1.86 $\pm$ 1.75 \\
ProxyOnly & 50 & 3.46 $\pm$ 4.88 & 1.59 $\pm$ 3.11 & 0.74 $\pm$ 0.18 & 0.16 $\pm$ 0.27 & 2.25 $\pm$ 3.23 \\
ProxyOnly & 100 & 3.50 $\pm$ 4.27 & 1.21 $\pm$ 1.70 & 0.77 $\pm$ 0.22 & 0.19 $\pm$ 0.28 & 1.80 $\pm$ 1.92 \\
ProxyOnly & 200 & 3.79 $\pm$ 4.53 & 2.05 $\pm$ 3.16 & 0.81 $\pm$ 0.18 & 0.22 $\pm$ 0.35 & 2.41 $\pm$ 3.09 \\
IPW-Transport & 25 & 2.30 $\pm$ 3.08 & 0.81 $\pm$ 1.22 & 0.79 $\pm$ 0.20 & 0.05 $\pm$ 0.10 & 1.43 $\pm$ 1.69 \\
IPW-Transport & 50 & 3.36 $\pm$ 5.84 & 1.38 $\pm$ 2.60 & 0.81 $\pm$ 0.20 & 0.16 $\pm$ 0.59 & 2.08 $\pm$ 3.58 \\
IPW-Transport & 100 & 4.87 $\pm$ 6.95 & 1.91 $\pm$ 3.63 & 0.76 $\pm$ 0.27 & 0.38 $\pm$ 0.87 & 2.91 $\pm$ 4.38 \\
\bottomrule
\end{tabular}%
}
\end{table}

\end{document}